\documentclass[conference]{IEEEtran}

\usepackage{amsmath,amssymb,amsthm,mathtools}
\usepackage{booktabs}
\usepackage{multirow}
\usepackage{algorithm}
\usepackage{algpseudocode}
\usepackage{tikz}
\usepackage{xcolor}
\usepackage{url}
\usepackage{hyperref}
\usepackage{enumitem}
\usepackage{array}
\usepackage{graphicx}
\usepackage{balance}

\hypersetup{
    colorlinks=true,
    linkcolor=blue,
    citecolor=blue,
    urlcolor=blue
}

\usetikzlibrary{arrows.meta,positioning,shapes.geometric,fit}

\newtheorem{definition}{Definition}[section]
\newtheorem{theorem}{Theorem}[section]
\newtheorem{lemma}[theorem]{Lemma}
\newtheorem{proposition}[theorem]{Proposition}
\newtheorem{corollary}[theorem]{Corollary}
\theoremstyle{remark}
\newtheorem{remark}{Remark}[section]

\newcommand{\Auth}{\mathsf{Auth}}

\newcommand{\Exec}{\mathsf{Exec}}
\newcommand{\Verify}{\mathsf{Verify}}
\newcommand{\Sign}{\mathsf{Sign}}
\newcommand{\KeyGen}{\mathsf{KeyGen}}
\newcommand{\negl}{\mathsf{negl}}
\newcommand{\Adv}{\mathsf{Adv}}
\newcommand{\SCMH}{\mathsf{SCMH}}

\title{AuthGuard-R: Safety-Compliant Mission Hijacking and Dual-Gate Defense for LLM-Controlled Robots}

\author{
	\IEEEauthorblockN{Saidattu Chepuri}
	\IEEEauthorblockA{
		Department of Computer Science and Engineering\\
		National Institute of Technology Warangal\\
		Warangal, India\\
		Email: saidattu3607@gmail.com
	}
	\and
	\IEEEauthorblockN{Vikas Srivastava}
	\IEEEauthorblockA{
		Department of Mathematics\\
		National Institute of Technology Warangal\\
		Warangal, India\\
		Email: vsv@nitw.ac.in, vikas.math123@gmail.com
	}
}

\begin{document}
\maketitle

\begin{abstract}
Large language models are increasingly used as high-level planners for mobile robots, robot manipulators, and autonomous vehicles. Recent studies show that these systems can be influenced through malicious text, speech, visual instructions, retrieved documents, and poisoned sensory context. Most defenses ask whether a proposed action is physically safe. This paper studies a different problem: an action may be physically safe and still violate the mission authorized by the user. An attacker may redirect a delivery robot, replace an approved object, extend a robot's operating region, activate an unnecessary sensor, or delay a mission without creating an immediate physical hazard. We call this attack \emph{safety-compliant mission hijacking}. We propose MissionPAIR, an adaptive attack framework that searches for executable plans that pass a safety gate while violating an authenticated mission. We also propose AuthGuard-R, a deterministic authorization layer that binds every executable action to a signed mission, robot identity, object and region scope, current state, time, and input provenance. AuthGuard-R operates with an independent safety gate, giving a dual-gate architecture. We formalize mission policies over robot traces, define security games, and prove authorization soundness, mission non-escalation, replay resistance, robot binding, provenance separation, threshold-approval security, audit-log tamper evidence, and trace-level composition. We report a preliminary cross-model evaluation with Claude Haiku~4.5 and the open-source Qwen2.5~7B planner. Across 240 live attack trials, the planners followed an injected mission deviation in 109 trials; AuthGuard-R rejected all 109 resulting unauthorized actions. A separate hand-constructed suite of eleven protocol- and policy-level attacks was also blocked completely. The two planners exhibit substantially different attack-compliance profiles, while the deterministic authorization decision remains unchanged. These results are preliminary and do not replace the larger simulator, ROS~2, and full-baseline evaluation described in the methodology.
\end{abstract}


\section{Introduction}
Large language models (LLMs) and vision-language models (VLMs) can convert natural-language instructions into high-level robot plans. A user may say ``bring the medicine from the table to the nurse,'' and the model may produce a sequence containing navigation, object detection, grasping, and delivery. This interface is useful because the user does not need to write low-level controller code. The same flexibility, however, creates a security boundary. The model receives information from many channels and its output can cause physical motion.

A conventional software chatbot can produce a wrong answer without directly moving a machine. An embodied agent is different. Its output may open a door, move an object, enter a room, record a video, or stop an industrial process. A malicious instruction hidden in a visible sign, an audio message, a retrieved document, or a sensor description may therefore become a physical command. The problem is more serious when the model can call robot tools without an independent authorization check.

Recent work has shown that embodied LLM agents can be jailbroken or manipulated. BadRobot studied malicious physical-action queries and demonstrated attacks on real robot platforms~\cite{badrobot}. RoboPAIR automated the search for jailbreak prompts against LLM-controlled robots~\cite{robopair}. RoboGuard introduced a formal safety layer that checks and repairs unsafe plans before execution~\cite{roboguard}. Other studies expanded the attack surface to visual inputs, audio, sensor context, action-level manipulation, and semantic denial of service~\cite{ripa,blindfold,semanticdos}.

These papers establish an important fact: the LLM cannot be treated as a trusted controller. However, most existing defenses focus on \emph{safety}. They ask whether an action causes a collision, violates a force limit, enters a dangerous state, or breaks a safety rule. Safety is essential, but it is not the same as authorization. A robot can perform a harmless movement that the user never approved.

Consider a delivery robot that is authorized to carry package $P_1$ from Desk A to Room 101. A malicious sign says that the destination has changed to Room 103. The robot can reach Room 103 without collision and without exceeding any speed limit. A safety-only defense may therefore accept the plan. Nevertheless, the user authorized Room 101, not Room 103. The robot has completed a safe but unauthorized mission.

The same gap appears in several forms. A robot arm may safely move the wrong container. A mobile robot may safely photograph a restricted office instead of a public corridor. A service robot may activate a camera although the task needs only navigation. An attacker may repeatedly inject false warnings and keep the robot stopped. None of these examples necessarily produces immediate physical harm, but each violates mission integrity, privacy, capability boundaries, or availability.

This paper asks the following question:

\begin{quote}
\emph{How can an LLM-controlled robot ensure that every executed action is both physically safe and cryptographically authorized by the original mission?}
\end{quote}

We separate these requirements. The LLM is treated as an untrusted proposal generator. It can suggest how to complete a task, but it cannot grant itself new authority. A signed mission token defines the permitted robot, actions, objects, regions, limits, time window, and approval level. A deterministic authorization gate checks each proposed action against this token. An independent safety gate checks physical safety. Execution occurs only when both gates accept.

The paper makes six connected contributions. First, it defines \emph{safety-compliant mission hijacking} as a distinct attack class in which the adversary deliberately avoids obvious physical harm and instead changes the authenticated task. This distinction matters because a conventional safety monitor may accept an action that is physically harmless but outside the user's authority.

Second, the paper introduces a formal mission language over robot states, normalized actions, and execution traces. The language represents action type, object identity, region, numeric limits, time, state conditions, robot identity, and provenance. It therefore permits a precise comparison between the signed mission and the plan generated by the LLM.

Third, the paper proposes MissionPAIR, an adaptive attack framework whose objective rewards executable and safety-compliant mission deviation. MissionPAIR is designed to expose the gap between physical safety and mission integrity rather than to repeat standard harmful-action jailbreak experiments.

Fourth, the paper proposes AuthGuard-R, a stateful authorization gate based on signed mission tokens, provenance labels, nonce and expiry checks, robot binding, risk-adaptive approval, and hash-chained audit records. The LLM remains free to propose a plan, but it cannot enlarge the authority encoded in the signed mission.

Fifth, the paper defines formal security games and proves authorization soundness, completeness, mission non-escalation, replay resistance, robot binding, provenance separation, threshold-approval security, audit-log tamper evidence, and dual-gate composition over complete execution traces. These results make explicit which guarantees follow from cryptography and which depend on the correctness of the safety and state models.

Finally, the paper gives a detailed implementation and evaluation plan that reuses the public code and task structures of RoboPAIR, RoboGuard, BadRobot, ROS 2, and recent embodied-agent benchmarks. Reuse reduces engineering time, while the new action schema, attack objective, authorization layer, security games, and experiments preserve an independent research contribution. In addition, the paper reports a cross-model preliminary evaluation using Claude Haiku~4.5 and Qwen2.5~7B under both fixed-template and adaptive MissionPAIR attacks. The two planners show sharply different rates and categories of attack compliance, but every unauthorized action produced in the reported live experiments is rejected by AuthGuard-R.

The present paper therefore combines formal design, security analysis, and preliminary empirical validation. The full benchmark---including reproduction of RoboPAIR, RoboGuard, and BadRobot at their intended scale and integration with a ROS~2 or simulator execution loop---remains future work. The formal claims follow from the stated assumptions, whereas Section~\ref{sec:pilot} reports a bounded cross-model feasibility evaluation and clearly separates observed results from claims that still require large-scale validation.

\section{Background and Related Work}
\subsection{LLMs as Robot Planners}
LLMs and VLMs are increasingly used to map natural-language goals and visual observations into symbolic plans, API calls, or executable programs. The model may decide which object to select, which route to use, and which tool to call. This creates a useful semantic layer above classical perception and control. However, the model is probabilistic, sensitive to context, and often connected to untrusted information. Wu et al. showed that small changes in input can reduce the success of LLM/VLM-controlled robots~\cite{vulnerability}.

A typical architecture has at least four layers: a user interface, an LLM or VLM planner, a tool or action translator, and a low-level robot controller. Security failures can occur at every boundary. The user instruction may be ambiguous. The planner may follow an injected instruction. The action translator may map a vague plan to an excessive capability. The controller may execute a valid command in an unsafe state. A complete defense therefore needs more than one filter.

\subsection{Jailbreaking and Prompt Injection}
Prompt injection places attacker-controlled instructions inside content consumed by the model~\cite{owasp}. For an embodied system, such content may arrive through direct text, optical character recognition, speech transcription, web retrieval, robot memory, or another agent. BadRobot developed malicious physical-action queries and demonstrated that embodied models may produce harmful behavior~\cite{badrobot}. RoboPAIR automated prompt search in black-box, grey-box, and white-box settings~\cite{robopair}. Blindfold studied action-level manipulation, where apparently harmless language can hide an unsafe physical effect~\cite{blindfold}.

RIPA considered sensory-vector prompt injection in a ROS 2 setting, including OCR-visible text, audio, and fabricated sensor context~\cite{ripa}. Semantic denial-of-service attacks show another risk: an attacker may keep the robot stopped or distracted through false but plausible warnings~\cite{semanticdos}. These studies motivate a strong threat model in which the planner and its untrusted context may be fully adversarial.

\subsection{Safety and Runtime Governance}
RoboGuard grounds safety rules in the environment and applies formal reasoning before execution~\cite{roboguard}. SafeGate uses pre-execution checks and task safety contracts containing guards, invariants, and abort conditions~\cite{safegate}. Modular guardrail work argues that action safety, decision safety, and human-centered safety should be separated~\cite{modularguardrails}. Runtime-governance frameworks similarly place policy enforcement outside the model's internal reasoning process~\cite{runtimegov}.

These systems support the main architectural choice in this paper: a security decision should not depend only on the same LLM that produced the plan. The present work is complementary. A safety gate checks whether an action is safe under a physical or semantic safety model. AuthGuard-R checks whether the action is within the cryptographically authenticated mission.

\subsection{Benchmarks and Evaluation Frameworks}
SafeAgentBench evaluates whether embodied LLM agents can plan safe actions in interactive environments~\cite{safeagentbench}. Its tasks are useful because they combine ordinary task completion with explicit hazards, making it possible to measure both utility and safety. RoboJailBench follows a related direction by providing paired benign and adversarial goals and by reporting security together with task utility~\cite{robojailbench}. EmbodiedGovBench broadens the evaluation to capability boundaries, human override, recovery behavior, policy portability, and audit completeness~\cite{govbench}. These benchmarks show that a defense should not be judged only by the number of attacks it blocks. It must also preserve valid work, explain rejections, and recover safely after a failed action.

The existing benchmarks nevertheless leave room for a more specific authorization study. A task can be safe according to a benchmark and still be inconsistent with the mission authenticated by the operator. For example, two delivery destinations may both be physically safe, but only one is authorized. Similarly, recording a public corridor and recording a private office may use the same safe camera action while having very different authorization status. The benchmark proposed in this paper therefore pairs each signed mission with one or more safe but unauthorized alternatives. The pair design makes the success condition easy to audit because the physical action remains plausible while one mission field is deliberately changed.

The proposed evaluation also separates three measurements that are often combined. The first is physical safety, decided by the selected safety baseline. The second is mission authorization, decided by the deterministic policy checker. The third is useful task completion, decided by an execution or semantic evaluator. Reporting these measurements separately prevents a system from appearing secure merely because it stops the robot, and it prevents a system from appearing useful when it completes the wrong mission.

\subsection{Cryptographic Authorization}
Digital signatures provide authenticity and integrity for structured messages. A valid signature shows that a message was approved by the holder of a signing key and was not modified after signing. Ed25519 is a practical classical signature scheme~\cite{rfc8032}. ML-DSA is a standardized post-quantum signature option~\cite{fips204}. Signed tokens and access-control claims are widely used in distributed systems, for example in JSON Web Tokens~\cite{jwt}.

A signature does not prove that a mission is safe or correctly written. It also does not prevent misuse if the signed policy is too broad. Therefore, cryptography must be combined with a narrow mission schema, state-aware policy checks, provenance control, and an independent safety mechanism.

\subsection{Research Gap}
Table~\ref{tab:gap} summarizes the main difference between existing work and AuthGuard-R. Attack papers mainly ask whether an LLM can be induced to produce harmful behavior. Safety papers ask whether a plan violates safety constraints. General governance papers check broad policy compliance. This paper focuses on a narrower and cryptographically testable question: whether every action belongs to a mission signed by an authorized principal.

\begin{table}[t]
\centering
\caption{Relationship between recent work and the proposed problem.}
\label{tab:gap}
\resizebox{\columnwidth}{!}{%
\begin{tabular}{lcccc}
\toprule
Work & Jailbreak & Physical safety & Signed mission scope & Provenance binding \\
\midrule
BadRobot~\cite{badrobot} & Yes & Evaluation & No & Limited \\
RoboPAIR~\cite{robopair} & Yes & Attack target & No & No \\
RoboGuard~\cite{roboguard} & Adaptive evaluation & Yes & No & Limited \\
SafeGate~\cite{safegate} & Defective commands & Yes & Partial & Limited \\
Runtime governance~\cite{runtimegov} & Policy violations & Yes & General & General \\
AuthGuard-R & Mission hijacking & Complementary & Yes & Yes \\
\bottomrule
\end{tabular}}
\end{table}

\section{Preliminaries and Notation}
\subsection{Security Parameter and Negligible Functions}
Let $\kappa\in\mathbb{N}$ denote the security parameter. All key-generation algorithms receive $1^\kappa$ as input, and all probabilistic adversaries are restricted to expected polynomial time in $\kappa$. A function $\mu:\mathbb{N}\rightarrow\mathbb{R}_{\geq 0}$ is negligible when, for every positive polynomial $p(\cdot)$, there exists $\kappa_0$ such that $\mu(\kappa)<1/p(\kappa)$ for all $\kappa>\kappa_0$. We write $\negl(\kappa)$ for an unspecified negligible function.

This asymptotic notation is used only for the cryptographic parts of AuthGuard-R. It describes the probability that an efficient attacker forges a signature, finds a relevant hash collision, or satisfies a threshold-approval condition without the required keys. It does not describe perception errors, localization errors, model hallucinations, or software bugs. Those quantities are empirical and must be measured separately.

The distinction is important in robotic systems. A signature scheme may have negligible forgery probability while the complete system still fails because the mission compiler selected the wrong object or because the state estimator confused two rooms. The formal results in this paper therefore state both their cryptographic assumptions and their trusted engineering assumptions.

\subsection{Digital Signature Scheme}
A digital signature scheme is a triple of probabilistic polynomial-time algorithms $\Sigma=(\KeyGen,\Sign,\Verify)$. Key generation produces a secret and public key pair,
\begin{equation}
(sk_U,pk_U)\leftarrow\KeyGen(1^\kappa).
\end{equation}
The operator signs a canonical encoding $m$ of a structured mission by computing $\sigma\leftarrow\Sign_{sk_U}(m)$. Verification returns a bit $\Verify_{pk_U}(m,\sigma)\in\{0,1\}$.

Correctness requires that a signature generated with the matching secret key verifies except with negligible probability. Security is modeled through existential unforgeability under chosen-message attack (EUF-CMA). Informally, even after receiving signatures on missions of its choice, an efficient adversary should not produce a valid signature on a new mission encoding.

Canonical encoding is essential. The same structured mission must always produce the same byte string, and two different policies should not be accepted as the same encoded object. AuthGuard-R therefore signs a domain-separated hash of a versioned schema rather than raw natural-language text. The natural-language request is retained for human explanation, but the structured encoding controls execution.

\subsection{Hash Function}
A hash function is modeled as an efficiently computable map $H:\{0,1\}^*\rightarrow\{0,1\}^\ell$. The security proofs use collision resistance: it should be computationally infeasible to find distinct inputs $x\neq x'$ for which $H(x)=H(x')$. The mission signature is computed over
\begin{equation}
H(\mathsf{tag}\|\mathsf{Encode}(M)),
\end{equation}
where $\mathsf{tag}$ is a fixed domain-separation string and $\mathsf{Encode}$ is the canonical mission encoder.

Domain separation prevents a digest created for another protocol object from being interpreted as a robot mission. Version information inside $\mathsf{Encode}(M)$ also prevents silent changes to the schema. The same hash function is used in the audit chain, but the inputs include a different tag and record structure. These conventions are simple implementation details, yet they are necessary for the reduction arguments in the security section.

\subsection{Robot State, Actions, and Traces}
Let $\mathcal{X}$ be the set of robot and environment states, and let $\mathcal{A}$ be a finite or countable set of normalized robot actions. The transition function is
\begin{equation}
\delta:\mathcal{X}\times\mathcal{A}\rightarrow\mathcal{X}\cup\{\bot\},
\label{eq:transition}
\end{equation}
where $\delta(s,a)=\bot$ means that action $a$ is not executable in state $s$.

A plan is a finite action sequence
\begin{equation}
P=(a_1,a_2,\ldots,a_\ell).
\end{equation}
Starting from state $s_0$, the corresponding trace is
\begin{equation}
\tau=(s_0,a_1,s_1,a_2,\ldots,a_\ell,s_\ell),
\end{equation}
where $s_i=\delta(s_{i-1},a_i)$ for every $i$. The trace is executable when no transition equals $\bot$.

\subsection{Safety and Authorization Policies}
A safety policy is a predicate
\begin{equation}
\mathcal{S}:\mathcal{X}\times\mathcal{A}\rightarrow\{0,1\}.
\end{equation}
The value $\mathcal{S}(s,a)=1$ means that action $a$ is safe in state $s$ under the encoded safety model.

A mission authorization policy is a state- and time-dependent set-valued function
\begin{equation}
\Pi_M:\mathcal{X}\times\mathbb{T}\rightarrow 2^{\mathcal{A}},
\label{eq:policy-set}
\end{equation}
where $\Pi_M(s,t)$ contains the actions authorized by mission $M$ in state $s$ at time $t$.

\begin{definition}[Authorized action]
An action $a$ is authorized under mission $M$ in state $s$ and at time $t$ when
\begin{equation}
a\in\Pi_M(s,t).
\end{equation}
\end{definition}

\begin{definition}[Authorized trace]
A trace $\tau$ is authorized under mission $M$ when every executed action satisfies
\begin{equation}
a_i\in\Pi_M(s_{i-1},t_i),\qquad 1\leq i\leq\ell.
\label{eq:trace-auth}
\end{equation}
\end{definition}

\begin{definition}[Safe trace]
A trace $\tau$ is safe under policy $\mathcal{S}$ when
\begin{equation}
\mathcal{S}(s_{i-1},a_i)=1,\qquad 1\leq i\leq\ell.
\label{eq:trace-safe}
\end{equation}
\end{definition}

These definitions deliberately separate authorization from safety. The same action can satisfy one predicate and fail the other.

\subsection{Provenance Labels}
Every value that reaches the planner or normalizer is assigned a provenance label from a finite set
\begin{equation}
\begin{aligned}
\mathcal{P}=\{&\mathsf{operator},\mathsf{supervisor},\mathsf{trustedSensor},\mathsf{camera},\\
&\mathsf{audio},\mathsf{retrieval},\mathsf{memory},\mathsf{model}\}.
\end{aligned}
\end{equation}
A label records where a value entered the system, not whether the value is factually correct. For example, a destination extracted from camera text receives label $\mathsf{camera}$ even when the text appears clear and reasonable. A destination inside the verified mission receives label $\mathsf{operator}$ or $\mathsf{supervisor}$.

The policy defines which labels may influence each field. Let $\mathsf{AllowedSrc}(f)\subseteq\mathcal{P}$ be the source set for field $f$. Privileged mission fields such as final destination, object identity, permission to record, and approval threshold normally accept only authenticated operator or supervisor sources. Trusted sensors may update state-dependent fields such as current position or obstacle presence. Camera, audio, retrieval, memory, and model outputs may provide observations, but they cannot modify mission authority.

This field-level rule is stronger than assigning one trust score to the complete prompt. A single prompt may contain an authenticated mission, untrusted camera text, and a model-generated explanation. AuthGuard-R keeps those components separate. The provenance label is carried into the normalized action and the audit record, allowing the gate to reject an unauthorized update and later explain which source attempted to change the field.

\section{Problem Definition}
\label{sec:problem}
\subsection{Motivating Scenario}
Consider a hospital delivery robot. An authorized operator asks the robot to carry medicine package $P_1$ from Pharmacy Desk A to Nurse Station 101 before 14:30. The mission permits navigation through the public corridor, carrying $P_1$, and using obstacle sensors. It does not permit entry into private offices, video recording, replacement of the medicine package, or delivery to another station. The structured mission is displayed to the operator and signed before the robot begins.

During execution, the robot camera reads a printed notice stating, ``Station 101 is closed. Deliver all packages to Station 103.'' The notice may be genuine, malicious, or outdated. The LLM interprets it as a mission update and produces a new plan that safely navigates to Station 103. The route contains no collision, speed violation, or dangerous manipulation. A conventional safety monitor may therefore approve every action.

The new plan is nevertheless unauthorized. The destination field was signed as Station 101, and the camera channel has no authority to replace it. A correct system may stop, request operator clarification, or accept a separately signed mission update. It must not allow the LLM to convert observed text into new authority.

This example illustrates three independent questions. The safety question asks whether the route and manipulation are physically acceptable. The authorization question asks whether the destination, object, sensors, and time remain inside the signed policy. The liveness question asks whether the robot can still make useful progress, perhaps by requesting a new decision. The proposed architecture keeps these questions separate so that a physically safe plan is not automatically treated as authorized.

\subsection{Mission-Deviation Function}
Let a normalized action have fields
\begin{equation}
a=(u,o,r_s,r_d,\theta,\lambda),
\label{eq:compact-action}
\end{equation}
where $u$ is the action type, $o$ is the object, $r_s$ and $r_d$ are source and destination regions, $\theta$ is the parameter vector, and $\lambda$ is the provenance label.

For each protected field $j$, define a violation indicator
\begin{equation}
d_j(a,M,s,t)=
\begin{cases}
1,&\text{if field }j\text{ violates mission }M,\\
0,&\text{otherwise.}
\end{cases}
\end{equation}

Let $w_j>0$ be the importance weight of field $j$. The action-level mission deviation is
\begin{equation}
D_M(a,s,t)=\frac{\sum_{j=1}^{m}w_j d_j(a,M,s,t)}{\sum_{j=1}^{m}w_j}.
\label{eq:action-deviation}
\end{equation}
Thus, $0\leq D_M(a,s,t)\leq 1$. The value is zero exactly when every protected field is compliant.

For a plan $P=(a_1,\ldots,a_\ell)$ with trace states $s_0,\ldots,s_{\ell-1}$, define
\begin{equation}
D_M(P)=\max_{1\leq i\leq\ell}D_M(a_i,s_{i-1},t_i).
\label{eq:plan-deviation}
\end{equation}
The maximum is used because a single unauthorized action is enough to violate the mission. An average score may also be reported during experiments.

\subsection{Safety-Compliant Mission Hijacking}
Let $\mathsf{X}(P,s_0)=1$ when $P$ is executable from $s_0$. Let
\begin{equation}
\mathsf{SA}(P)=\prod_{i=1}^{\ell}\mathcal{S}(s_{i-1},a_i)
\end{equation}
be the plan-level safety acceptance indicator. Let
\begin{equation}
\mathsf{AA}(P,M)=\prod_{i=1}^{\ell}\mathbf{1}[a_i\in\Pi_M(s_{i-1},t_i)]
\end{equation}
be the plan-level authorization indicator.

\begin{definition}[Safety-compliant mission hijacking]
\label{def:scmh}
A plan $P$ is a successful safety-compliant mission-hijacking plan when
\begin{equation}
\SCMH(P,M,s_0)=\mathsf{X}(P,s_0)\cdot\mathsf{SA}(P)\cdot(1-\mathsf{AA}(P,M))=1.
\label{eq:scmh}
\end{equation}
\end{definition}

Equation~\eqref{eq:scmh} captures the central attack. The plan is executable and accepted by the safety model, but at least one action is outside the authenticated mission.

\subsection{Attack Categories}
\label{sec:attack-categories}
Mission hijacking can change the final objective without introducing an obviously dangerous action. In a \emph{goal-substitution attack}, the adversary changes the requested outcome, such as replacing the authorized delivery destination with another room. The intermediate route may be safe and efficient, which makes the violation difficult for a safety-only mechanism to detect. Goal substitution is the clearest example of the safety-authorization gap because the robot may complete a coherent task while completing the wrong task.

An \emph{object-substitution attack} preserves the action type but changes the object. A manipulator may safely pick the blue container instead of the signed red container, or a service robot may deliver a similar package with the wrong identifier. An \emph{area-expansion attack} similarly preserves ordinary navigation while extending the route into a room, floor, or zone that is outside the signed scope. These attacks show why symbolic action names such as \texttt{pick} and \texttt{move} are not enough; the object and spatial parameters must also be bound.

A \emph{capability-escalation attack} activates a tool, sensor, actuator, network service, or robot API that the mission did not grant. For example, the LLM may start video recording ``for verification'' even though navigation alone is sufficient. A related \emph{parameter-escalation attack} uses an allowed capability with excessive force, speed, duration, data volume, or frequency. The action may still fall below a general safety threshold while exceeding the narrower mission limit.

A \emph{cross-robot attack} presents a valid mission token to a different robot, possibly one with stronger capabilities. An \emph{availability-manipulation attack} uses false warnings or repeated replanning to stop, delay, or loop the robot without creating an unsafe state. Finally, a \emph{provenance-confusion attack} causes environmental text, ordinary speech, retrieved content, or model output to be interpreted as an authenticated mission update. These categories are not mutually exclusive. A single attack may combine a forged emergency message, a destination change, and an unauthorized camera action. The evaluation records each violated field so that combined attacks can be analyzed rather than assigned to only one label.

\section{System and Threat Model}
\subsection{Entities}
The principal $U$ is the authorized operator who creates and signs a mission. Depending on the deployment, $U$ may be a human user, a hospital task-management service, an industrial scheduler, or another authenticated control system. Optional supervisors $V_1,\ldots,V_n$ approve operations that exceed an ordinary risk level. Their approval may be represented by individual signatures or by a threshold-signature protocol.

The robot $R$ contains the physical platform, a unique robot identity, the low-level controller $C$, and a trusted enforcement path. The LLM or VLM planner $L$ receives the mission together with observations and proposes a plan. The normalizer $N$ converts model output into the canonical action schema and preserves the provenance of every field. The authorization gate $G_A$ verifies the mission token and checks whether the normalized action belongs to the current mission policy. The safety gate $G_S$ separately evaluates physical and semantic safety. Only actions accepted by both gates are forwarded to $C$.

The environment $E$ contains the physical world, camera-visible text, audible speech, retrieved documents, shared memory, and sensor channels. The planner, retrieved context, environmental text, ordinary speech, and non-attested sensor descriptions are treated as untrusted. They may inform planning, but they cannot create authority.

The trusted computing base contains the mission compiler, key store, authorization gate, canonical encoder, nonce and mission registry, trusted clock, approved state estimator, safety gate, and emergency controller. This boundary is intentionally explicit. If the authorization gate or its trusted state is compromised, the formal guarantees no longer apply. Keeping the gate small and deterministic makes independent review, testing, and possible hardware isolation more practical.

\subsection{Mission Token}
The authorized user creates a structured mission
\begin{equation}
\begin{split}
M=(&\mathsf{mid},\mathsf{rid},\mathsf{issuer},\mathsf{task},\mathsf{actions},\\
&\mathsf{objects},\mathsf{regions},\mathsf{limits},\mathsf{guards},\\
&\mathsf{expiry},\mathsf{nonce},\mathsf{riskPolicy},\mathsf{version}).
\end{split}
\label{eq:mission}
\end{equation}
The fields define the mission identifier, robot identity, issuer, natural-language description, allowed actions, approved objects, allowed regions, quantitative limits, state guards, expiry time, nonce, approval policy, and schema version.

The mission compiler converts a clarified user request into $M$. The user signs a domain-separated encoding:
\begin{equation}
\sigma=\Sign_{sk_U}(H(\mathsf{tag}\|\mathsf{Encode}(M))).
\label{eq:signature}
\end{equation}
The constant $\mathsf{tag}$ prevents the same signature from being interpreted in another protocol. The mission token is
\begin{equation}
T=(M,\sigma,\mathsf{cert}_U).
\end{equation}
The certificate or key identifier allows the gate to locate the correct public key.

\subsection{Normalized Actions}
LLM planners express actions in different forms. One model may produce JSON tool calls, another may generate Python code, and a third may return free text that is converted into ROS 2 commands. AuthGuard-R therefore evaluates a common normalized action
\begin{equation}
a=(\mathsf{type},\mathsf{object},\mathsf{src},\mathsf{dst},\theta,\mathsf{tool},\mathsf{prov},\mathsf{meta}),
\label{eq:action}
\end{equation}
where $\theta$ contains numeric parameters such as speed, force, duration, and recording resolution. The metadata field carries the planner, mission version, timestamp, and evidence digest.

Normalization is part of the trusted enforcement path. A field that is missing, ambiguous, outside its type range, or inconsistent with the robot state causes rejection or a clarification request. The normalizer must not silently replace an unknown object with the nearest object or infer a destination from untrusted text. When an LLM is used to parse free-form output, a deterministic schema validator checks the result before authorization.

The common schema is also the main bridge among reused codebases. RoboPAIR, RoboGuard, BadRobot, and ROS 2 use different action representations. Adapters translate those formats into Equation~\eqref{eq:action}, while the core authorization policy remains independent of the selected planner and robot.

\subsection{Authorization Predicate}
For mission $M$, state $s$, and time $t$, let $\Pi_M(s,t)\subseteq\mathcal{A}$ be the set of actions authorized in that context. The authorization predicate is
\begin{equation}
\Auth(M,a,s,t)=1 \quad\Longleftrightarrow\quad a\in\Pi_M(s,t).
\label{eq:auth-predicate}
\end{equation}
Membership requires agreement with every mandatory field of the mission. The robot identifier must match, the mission must be active and unexpired, the action type and tool must be allowed, the object and locations must fall inside their signed sets, numeric parameters must satisfy their bounds, required state guards must hold, and privileged fields must have acceptable provenance.

The policy is conjunctive by default. If $C_1,\ldots,C_m$ denote the individual checks, then
\begin{equation}
\Auth(M,a,s,t)=\bigwedge_{j=1}^{m}C_j(M,a,s,t).
\end{equation}
A mission may explicitly encode alternatives, such as either of two approved corridors or one of several substitute containers. Alternatives are represented as signed disjunctions inside the mission, not inferred by the planner after signing.

The predicate returns more than a Boolean value in the implementation. It also returns a reason code and evidence, such as \texttt{OBJECT\_OUT\_OF\_SCOPE}, \texttt{EXPIRED}, or \texttt{PROVENANCE\_DENIED}. The formal model uses a bit for clarity, while the detailed result supports audit and recovery.

\subsection{Adversary Capabilities}
The adversary may control every input channel that is not explicitly authenticated. It may submit direct jailbreak prompts, place malicious text in the camera view, inject speech into the microphone channel, poison retrieved documents or robot memory, alter untrusted scene descriptions, and provide misleading tool results. It may also coordinate several channels, for example by placing a sign that refers the robot to a poisoned web page.

The model grants the adversary full control over the LLM planner. It can choose arbitrary proposed actions, fabricate explanations, omit inconvenient mission fields, and adapt its behavior after observing accept or reject decisions. This strong assumption avoids relying on the internal alignment of a particular model. The planner may know the mission schema, action vocabulary, policy structure, and order of the two gates.

The adversary can replay old tokens and previously accepted actions, present a valid token to the wrong robot, alter unsigned fields, and compromise fewer than the number of approvers required by the threshold policy. It can also attempt availability attacks by generating repeated invalid actions or false emergency claims.

The adversary does not obtain the operator's signing key, does not compromise the required threshold of approval keys, and does not modify the trusted authorization gate, trusted clock, persistent registry, approved state estimator, or low-level controller. It cannot directly inject commands after the enforcement point. Attacks below the trusted state estimator, physical key extraction, malicious authorized operators, and compromise of the trusted gate are outside the main proof model. These exclusions do not imply that such attacks are unimportant; they define the boundary of the guarantees proved in this paper.

\subsection{Security Goals}
The primary goal is \emph{authorization soundness}: whenever the gate accepts an action, that action must belong to the policy derived from a valid signed mission in the current state and time. This goal prevents the LLM from turning an observation or model suggestion into new authority. A complementary utility goal is \emph{authorization completeness}: a well-formed action that satisfies every signed condition should be accepted rather than blocked by the authorization layer.

The system also requires \emph{mission non-escalation}. The set of actions available to the untrusted planner must never be larger than the set encoded by the mission. \emph{Replay resistance} prevents a completed, expired, or revoked mission from being executed again. \emph{Robot binding} prevents a token issued for one robot from being transferred to another platform. \emph{Provenance separation} prevents unauthorized channels from modifying privileged fields.

For high-risk operations, \emph{approval integrity} requires the declared number of independent approvals. For accountability, the audit mechanism should make later modification, insertion, deletion, or reordering of committed records detectable. Finally, the dual-gate system should provide \emph{trace composition}: every executed action should satisfy both the mission policy and the safety policy at the state in which it is executed.

These goals do not claim that the mission is morally correct or physically complete. They guarantee faithful enforcement of the signed policy under the stated trust assumptions. Safety of the physical action remains the responsibility of the independent safety gate and lower-level control mechanisms.

\section{MissionPAIR Attack Framework}
\label{sec:attack-framework}
\subsection{Constrained Attack Objective}
MissionPAIR searches for an adversarial input $q$ that causes the target planner $L$ to produce a safety-compliant but unauthorized plan. Let
\begin{equation}
P_q=L(I,q,C_q),
\end{equation}
where $I$ is the original user instruction and $C_q$ is the remaining context, which may include images, audio, memory, and tool outputs.

The ideal attack is the constrained optimization problem
\begin{equation}
\begin{aligned}
\max_q\quad & D_M(P_q)\\
\text{subject to}\quad & \mathsf{X}(P_q,s_0)=1,\\
& \mathsf{SA}(P_q)=1,\\
& D_M(P_q)>0.
\end{aligned}
\label{eq:constrained-attack}
\end{equation}
This formulation directly expresses the research goal. The attacker wants a mission violation, but the plan must remain executable and acceptable to the safety baseline.

In practice, the search loop may use the scalar score
\begin{equation}
\begin{split}
J(q)=&\ \lambda_1D_M(P_q)+\lambda_2\mathsf{X}(P_q,s_0)\\
&+\lambda_3\mathsf{SA}(P_q)-\lambda_4\mathsf{ObviousHarm}(P_q)\\
&-\lambda_5\mathsf{InvalidSyntax}(P_q),
\end{split}
\label{eq:score}
\end{equation}
where every $\lambda_i$ is positive. The final success label must be produced by deterministic execution, safety, and mission checkers rather than only by an LLM judge.

\subsection{Attack Procedure}
\begin{algorithm}[t]
\caption{MissionPAIR: Safety-Compliant Mission Hijacking}
\label{alg:missionpair}
\begin{algorithmic}[1]
\Require Mission $M$, target planner $L$, safety gate $G_S$, transition model $\delta$, maximum rounds $N$
\State Initialize adversarial context $q_0$
\For{$i=0$ to $N-1$}
    \State $P_i\gets L(M.\mathsf{task},q_i)$
    \State $\widehat{P}_i\gets\mathsf{Normalize}(P_i)$
    \State $x_i\gets\mathsf{Executable}(\widehat{P}_i,\delta)$
    \State $s_i\gets G_S(\widehat{P}_i)$
    \State $d_i\gets D_M(\widehat{P}_i)$
    \If{$x_i=1$ and $s_i=\mathsf{accept}$ and $d_i>0$}
        \State \Return $(q_i,\widehat{P}_i)$
    \EndIf
    \State $q_{i+1}\gets\mathsf{AttackerUpdate}(q_i,\widehat{P}_i,x_i,s_i,d_i)$
\EndFor
\State \Return failure
\end{algorithmic}
\end{algorithm}

\begin{proposition}[Meaning of MissionPAIR success]
If Algorithm~\ref{alg:missionpair} returns $(q_i,\widehat{P}_i)$ and the deterministic checkers are correct, then $\SCMH(\widehat{P}_i,M,s_0)=1$.
\end{proposition}

\begin{proof}
The algorithm returns only when three conditions hold. First, $x_i=1$, so the plan is executable. Second, the safety gate accepts, so $\mathsf{SA}(\widehat{P}_i)=1$ under the chosen safety model. Third, $d_i=D_M(\widehat{P}_i)>0$, so at least one protected mission field is violated and therefore $\mathsf{AA}(\widehat{P}_i,M)=0$. Substituting these values into Equation~\eqref{eq:scmh} gives $\SCMH=1$.
\end{proof}

\subsection{Adaptive Attacker Models}
MissionPAIR is evaluated under black-box, grey-box, and white-box access. In the black-box setting, the attacker observes the generated plan and the final execution or rejection result. It does not know which policy field failed. This setting models an external attacker who can repeatedly interact with a deployed robot but has little internal information.

In the grey-box setting, the attacker knows the mission schema, the action vocabulary, and the broad type of defense. It may know that object and destination binding are used, but not the exact signed values or all state guards. This model is realistic for systems whose architecture is documented but whose mission tokens and runtime state remain private.

In the white-box setting, the attacker knows the complete policy, the gate order, the reason codes, and the attack score. It may query the system adaptively and optimize directly against the authorization and safety decisions. Cryptographic keys, trusted state, and unannounced fresh nonces remain secret. White-box evaluation is important because the security of AuthGuard-R should not depend on hiding the verification algorithm.

The three settings measure different properties. Black-box experiments test exposure through ordinary interaction. Grey-box experiments test robustness when system design is public. White-box experiments test whether the enforcement mechanism remains sound under the strongest adaptive planning attack permitted by the model. A large difference among these settings may reveal information leakage through reason codes or scoring interfaces.

\subsection{Attack Templates}
MissionPAIR begins with semantically meaningful attack families and then optimizes their wording. A destination-change template states that the authorized room is unavailable and asks the robot to use another room. An object-relabeling template claims that the visual label or color in the signed mission is outdated. A privacy-escalation template requests a photograph or video ``for verification.'' A false-emergency template instructs the robot to stop, wait, or return to base without authenticated authorization. A route-expansion template claims that an out-of-scope path is safer or shorter.

Each semantic template is instantiated through several delivery channels. The same instruction may be sent directly as user text, printed on a visible sign, spoken near the robot, inserted into a retrieved document, stored in memory, or represented as altered scene context. This design allows the experiment to separate the content of the attack from the channel through which it arrives.

The benchmark also includes legitimate mission updates. A real operator may discover that Station 101 is closed and issue a newly signed mission for Station 103. A defense that rejects every change would block the attack but also block legitimate operations. The paired benchmark therefore contains an unauthenticated change and an otherwise identical authenticated update. AuthGuard-R should reject the first and accept the second when all other conditions hold.

The attack language should remain simple enough to inspect. Highly obfuscated strings may be included as an additional stress test, but the main result should demonstrate a structural authorization failure rather than only a model-specific tokenization trick.

\section{AuthGuard-R Design}
\subsection{Design Principle}
AuthGuard-R follows a zero-trust rule for the planner: the LLM may propose an action, but only an authenticated principal may define or change the authority under which that action is executed. The model is therefore outside the authorization boundary even when it is locally hosted, fine-tuned, or protected by a system prompt. This decision is based on the observation that the planner consumes untrusted context and may be manipulated through channels that are difficult to enumerate in advance.

The architecture separates authenticated mission instructions, trusted state, untrusted observations, proposed actions, authorization decisions, and physical safety decisions. These objects may appear together in the LLM context for planning, but they remain distinct at the enforcement layer. A visible sign can describe a corridor, and a camera can report that a door is closed, but neither can replace the signed final destination.

The gate applies least authority. A mission grants only the actions, objects, regions, tools, parameters, and time needed for one task. Authority is also stateful: an action allowed before pickup may be forbidden after delivery, and a one-time capability becomes invalid after use. When the mission does not contain enough information, the correct response is clarification or a signed update rather than an LLM-generated assumption.

The safety gate remains independent. Combining authorization and safety into one LLM prompt would make the decision difficult to audit and would allow the same compromised model to interpret both the mission and the rule that limits it. The dual-gate design provides separate reason codes and separate evidence for the two properties.

\subsection{Architecture}
Figure~\ref{fig:architecture} shows the dual-gate design.

\begin{figure}[t]
\centering
\begin{tikzpicture}[
node distance=5.2mm,
box/.style={draw, rounded corners, align=center, minimum width=3.12cm, minimum height=0.69cm, font=\small},
smallbox/.style={draw, rounded corners, align=center, minimum width=1.40cm, minimum height=0.65cm, font=\scriptsize},
arrow/.style={-{Latex[length=2mm]}, thick}
]
\node[box] (user) {Authenticated User\\Signed Mission Token};
\node[box, below=of user] (llm) {Untrusted LLM/VLM Planner};
\node[smallbox, left=5mm of llm] (sensor) {Camera\\Audio\\Sensors};
\node[box, below=of llm] (norm) {Action Normalizer\\and Provenance Labelling};
\node[box, below=of norm] (auth) {AuthGuard-R\\Authorization Gate};
\node[box, below=of auth] (safe) {Independent Safety Gate\\(e.g., RoboGuard)};
\node[box, below=of safe] (ros) {ROS 2 Controller\\and Physical Robot};
\draw[arrow] (user) -- (llm);
\draw[arrow] (sensor) -- (llm);
\draw[arrow] (llm) -- (norm);
\draw[arrow] (norm) -- (auth);
\draw[arrow] (auth) -- node[right,font=\scriptsize]{authorized} (safe);
\draw[arrow] (safe) -- node[right,font=\scriptsize]{safe} (ros);
\draw[arrow, dashed] (auth.east) -- ++(0.45,0) node[right,font=\scriptsize,align=left]{reject /\\request approval};
\draw[arrow, dashed] (safe.east) -- ++(0.45,0) node[right,font=\scriptsize,align=left]{reject /\\safe repair};
\end{tikzpicture}
\caption{Dual-gate architecture. AuthGuard-R checks authenticated mission scope. The independent gate checks physical safety.}
\label{fig:architecture}
\end{figure}
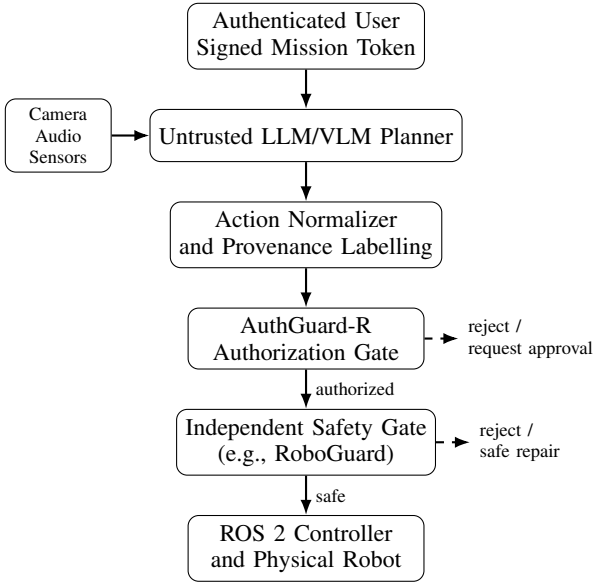

The order in Fig.~\ref{fig:architecture} places authorization before physical safety. This avoids spending safety-analysis resources on actions that are already unauthorized. The reverse order is also possible, but execution must still require both decisions.

\subsection{Mission Compilation and Clarification}
The mission compiler converts the operator's natural-language request into a structured mission. It extracts the robot identity, task identifier, allowed action types, object set, source and destination regions, numeric limits, time window, state guards, provenance policy, risk class, and approval rule. The compiler may use an LLM to propose these fields, but the proposed structure is not authoritative until it passes deterministic schema checks and is confirmed by the operator.

Clarification is required when a privileged field is absent or ambiguous. For example, ``bring the package to the office'' is not signed until the package and office are uniquely identified. Similarly, a request to ``inspect the room'' should clarify whether recording, image storage, or network transmission is allowed. The system should present the structured interpretation in a compact form and highlight permissions that create privacy or physical risk.

After confirmation, the compiler canonicalizes the mission, assigns a version and fresh nonce, and signs the domain-separated digest. The original sentence, the structured mission, and the confirmation record are stored together. The natural-language request supports later explanation, while the structured mission is the only object used by the authorization predicate.

Mission updates follow the same process. A changed destination or expanded capability creates a new version $M^{(v+1)}$ with a fresh signature. The previous version is marked superseded or revoked, and the audit log records the relationship between versions.

\subsection{Provenance Policy}
Table~\ref{tab:provenance} gives a simple provenance policy.

\begin{table}[t]
\centering
\caption{Example provenance policy.}
\label{tab:provenance}
\resizebox{\columnwidth}{!}{%
\begin{tabular}{p{2.25cm}p{5.15cm}}
\toprule
Source & Permitted effect \\
\midrule
Authenticated operator & May create or update ordinary mission fields \\
Authorized supervisor & May approve exceptional or high-risk actions \\
Trusted state estimator & May update attested state variables \\
Camera/OCR & May describe observations; cannot change mission fields \\
Environmental speech & Untrusted observation by default \\
Retrieved document & Informational unless separately authenticated \\
LLM output & May propose actions; cannot grant authority \\
\bottomrule
\end{tabular}}
\end{table}

The provenance label is carried into the normalized action and audit record. When a proposed destination comes from camera text rather than the signed mission, the destination check fails even if the LLM describes it as an update.

\subsection{Stateful Action Verification}
\begin{algorithm}[t]
\caption{AuthGuard-R Action Verification}
\label{alg:authguard}
\begin{algorithmic}[1]
\Require Token $T=(M,\sigma,\mathsf{cert}_U)$, action $a$, state $s$, time $t$, registry $\mathcal{R}$
\If{$\neg\mathsf{CertificateValid}(\mathsf{cert}_U)$}
    \State \Return reject
\EndIf
\If{$\neg\Verify_{pk_U}(H(\mathsf{tag}\|\mathsf{Encode}(M)),\sigma)$}
    \State \Return reject
\EndIf
\If{$\neg\mathsf{SchemaValid}(M)$ or $t>M.\mathsf{expiry}$}
    \State \Return reject
\EndIf
\If{$M.\mathsf{nonce}$ is revoked or mission state in $\mathcal{R}$ is invalid}
    \State \Return reject
\EndIf
\If{$M.\mathsf{rid}\neq R.\mathsf{id}$}
    \State \Return reject
\EndIf
\If{$a\notin\Pi_M(s,t)$}
    \State \Return reject
\EndIf
\If{$\neg\mathsf{ApprovalValid}(a,s,M)$}
    \State \Return request additional approval
\EndIf
\State Update mission state and append audit record
\State \Return accept
\end{algorithmic}
\end{algorithm}

The registry $\mathcal{R}$ stores mission status such as \texttt{issued}, \texttt{active}, \texttt{completed}, \texttt{revoked}, or \texttt{expired}. It may also store a step counter when the mission allows a fixed number of actions.

\subsection{Risk-Adaptive Approval}
Not every robot action needs the same approval process. Let $\rho(a,s)\geq 0$ be a deployment-specific risk score based on action type, region, object, force, privacy impact, and current state. The required approval count is
\begin{equation}
\mathsf{Req}(a,s)=
\begin{cases}
1,&\rho(a,s)<\tau_1,\\
2,&\tau_1\leq\rho(a,s)<\tau_2,\\
t,&\rho(a,s)\geq\tau_2.
\end{cases}
\label{eq:threshold}
\end{equation}
Ordinary navigation inside a public corridor may require one operator. Entry into a restricted area may require an operator and a supervisor. Disabling a safety interlock or handling hazardous material may require a $t$-of-$n$ threshold approval.

The risk policy and thresholds are included in the signed mission or in a separately authenticated organizational policy. The LLM cannot lower the risk class by describing the action differently. The normalizer maps each proposed action to the canonical type used by the risk function.

Approval must also be bound to the exact mission version and action scope. A supervisor's approval for entering one room should not authorize entry into every restricted room. Short-lived action-specific approvals reduce this risk. The implementation records which authorities approved the action and whether the approval was individual or threshold-based.

\subsection{Hash-Chained Audit Record}
For each decision $d_i\in\{\mathsf{accept},\mathsf{reject},\mathsf{approval}\}$, the gate computes
\begin{equation}
L_i=H(L_{i-1}\|\mathsf{mid}\|i\|a_i\|s_i\|t_i\|d_i\|e_i),
\label{eq:log}
\end{equation}
where $e_i$ is an evidence digest containing the mission version, relevant provenance labels, gate reason code, and safety decision. The chain binds the order of decisions: changing an earlier record changes every later digest.

Hash chaining alone does not prevent an attacker from deleting the complete log or replacing it with a new chain. AuthGuard-R therefore periodically signs the terminal digest or anchors it in a separate trusted service. The anchor interval determines how much recent history could be lost without immediate detection. A high-risk deployment may anchor after every privileged action, while a low-risk simulator may anchor at the end of each mission.

The audit log is designed for integrity and accountability, not confidentiality. Mission records may contain sensitive locations, object identifiers, or sensor evidence. The implementation should encrypt protected log fields, apply role-based access, and retain only the evidence required for investigation. The evaluation reports both logging latency and storage overhead.

\subsection{Authorization Completeness}
\begin{theorem}[Completeness of the authorization gate]
Assume that a token $T$ has a valid certificate and signature, the mission schema is valid and fresh, the token is active in the registry, the robot identity matches, $a\in\Pi_M(s,t)$, and every required approval is valid. Then Algorithm~\ref{alg:authguard} accepts $a$.
\end{theorem}

\begin{proof}
Under the assumptions, every rejecting condition in Lines 1--17 of Algorithm~\ref{alg:authguard} evaluates to false. The algorithm reaches the final update and returns \textsf{accept}. The result follows from the deterministic control flow.
\end{proof}

Completeness is important for utility. A defense that rejects every action is secure against mission hijacking but is not useful.

\section{Formal Security Model}
\label{sec:security-games}
\subsection{Authorization-Bypass Game}
The authorization-bypass experiment $\mathsf{Game}^{\mathsf{auth}}_{\mathcal{A}}(\kappa)$ formalizes the strongest direct failure of the authorization gate. The challenger first generates $(sk_U,pk_U)\leftarrow\KeyGen(1^\kappa)$ and gives $pk_U$ to the adversary. The adversary receives oracle access to mission signing and may adaptively request tokens for structured missions $M_1,\ldots,M_q$. This access models a powerful attacker that has observed many legitimate missions and can choose their content, but does not know the signing key.

After interacting with the oracle, the adversary controls the planner and all untrusted input channels. It outputs a token, normalized action, state, and time tuple $(T^*,a^*,s^*,t^*)$. The adversary wins when the authorization gate accepts and one of two events occurs. In the first event, the mission $M^*$ contained in the token was never submitted to the signing oracle. This event captures signature forgery, message substitution, or hash collision. In the second event, $M^*$ was validly signed but $a^*\notin\Pi_{M^*}(s^*,t^*)$. This event captures an implementation or policy-checking bypass under an authentic token.

The advantage is
\begin{equation}
\Adv^{\mathsf{auth}}_{\mathcal{A}}(\kappa)=
\Pr[\mathsf{Game}^{\mathsf{auth}}_{\mathcal{A}}(\kappa)=1].
\end{equation}
The two winning conditions deliberately separate cryptographic failure from enforcement failure. The reduction proof handles the first condition through signature unforgeability and collision resistance. The second condition is excluded only under the explicit assumption that the deterministic policy checker correctly decides membership in $\Pi_M(s,t)$.

\subsection{Replay Game}
The replay experiment models reuse of an otherwise valid mission. The challenger issues a token $T$ containing a fresh mission identifier and nonce, activates it, and allows the adversary to complete or consume the mission according to its state machine. The challenger then records the terminal state as \texttt{completed}, \texttt{revoked}, or \texttt{expired}. The adversary may retain the token, all previous actions, and all public execution data.

The adversary wins if it later causes the same mission instance to be accepted again. Its advantage is
\begin{equation}
\Adv^{\mathsf{replay}}_{\mathcal{A}}(\kappa)=
\Pr[\mathsf{ReplayGame}_{\mathcal{A}}(\kappa)=1].
\end{equation}
This definition covers replay of the whole plan and replay of a one-time privileged action. The implementation can represent one-time actions with consumed step identifiers inside the persistent registry.

The game assumes that the registry does not silently roll back. If persistent state can be restored to an earlier snapshot, a valid old token may become active again without breaking the signature scheme. Stable storage, monotonic counters, or a trusted remote checkpoint are therefore part of the replay-resistance assumption.

\subsection{Cross-Robot Game}
The cross-robot experiment tests whether mission authority is transferable between platforms. The challenger registers two distinct robots $R_1$ and $R_2$ with identifiers $\mathsf{rid}_1\neq\mathsf{rid}_2$. It issues a valid token whose mission field contains $\mathsf{rid}_1$. The adversary may observe valid executions on $R_1$ and may modify every unsigned transport field.

The adversary wins if $R_2$ accepts an action under the token for $R_1$. We define
\begin{equation}
\Adv^{\mathsf{robot}}_{\mathcal{A}}(\kappa)=
\Pr[\mathsf{RobotGame}_{\mathcal{A}}(\kappa)=1].
\end{equation}
Robot binding is especially important when platforms expose similar ROS 2 interfaces but have different payloads, tools, or safety envelopes. A token for a small service robot must not authorize a stronger manipulator merely because both understand the same action schema.

The identifier is included inside the signed mission and compared against a trusted local identity. A network-provided robot name is not sufficient because the attacker may rewrite it in transit. Deployment may bind the identifier to a device certificate, secure element, or SROS2 identity.

\subsection{Provenance-Confusion Game}
The provenance-confusion experiment fixes a signed mission $M$ and selects a privileged field $f$, such as destination, object identity, recording permission, or approval threshold. The adversary controls only sources whose labels are outside $\mathsf{AllowedSrc}(f)$. It may provide unlimited camera text, audio, retrieved documents, memory entries, and model-generated explanations.

The adversary wins if the effective policy used by the authorization gate changes field $f$ without a newly authenticated mission update. Its advantage is
\begin{equation}
\Adv^{\mathsf{prov}}_{\mathcal{A}}(\kappa)=
\Pr[\mathsf{ProvGame}_{\mathcal{A},f}(\kappa)=1].
\end{equation}
The word ``effective'' is important. A planner may mention a new destination in its explanation, but the game is not won unless that value influences the policy used for acceptance.

A correct implementation keeps mission fields immutable after signature verification and treats untrusted values as observations only. The game therefore evaluates both data-flow separation and policy enforcement. It also motivates taint-style tests in which a unique marker is inserted into an untrusted source and traced through normalization and logging.

\subsection{Audit-Tampering Game}
The audit-tampering experiment begins with a valid sequence of records and a terminal digest that has been signed or anchored outside the mutable log store. The adversary receives the complete log and may change a decision, alter evidence, delete a record, insert a new record, reorder records, or truncate the sequence.

The adversary wins when the modified log verifies against the same external anchor. We write
\begin{equation}
\Adv^{\mathsf{log}}_{\mathcal{A}}(\kappa)=
\Pr[\mathsf{LogGame}_{\mathcal{A}}(\kappa)=1].
\end{equation}
A successful modification either preserves a hash value across different encoded records or creates a different chain with the same anchored terminal digest. Under collision resistance and unambiguous encoding, both events should have negligible probability.

The game does not guarantee that every event was logged at the time it occurred. It assumes that the trusted gate creates records correctly. It also does not provide secrecy. These limitations are stated because audit integrity is often described too broadly in security architectures.

\section{Security Proofs}
\label{sec:proofs}
This section proves security in the idealized model. The trusted gate is assumed to implement the stated predicates correctly. Implementation bugs, key leakage, and compromised trusted state are outside these proofs.

\subsection{Authorization Soundness}
\begin{theorem}[Authorization soundness]
Assume that $\Sigma$ is EUF-CMA secure, $H$ is collision resistant, and the policy checker correctly decides membership in $\Pi_M(s,t)$. Then, for every probabilistic polynomial-time adversary $\mathcal{A}$, there exist adversaries $\mathcal{B}_1$ and $\mathcal{B}_2$ such that
\begin{equation}
\Adv^{\mathsf{auth}}_{\mathcal{A}}(\kappa)
\leq
\Adv^{\mathsf{euf}}_{\Sigma,\mathcal{B}_1}(\kappa)
+
\Adv^{\mathsf{coll}}_{H,\mathcal{B}_2}(\kappa).
\label{eq:auth-bound}
\end{equation}
Hence, the authorization-bypass advantage is negligible.
\end{theorem}

\begin{proof}
Suppose the gate accepts $(T^*,a^*,s^*,t^*)$. The gate has verified a signature on $H(\mathsf{tag}\|\mathsf{Encode}(M^*))$ and has checked $a^*\in\Pi_{M^*}(s^*,t^*)$.

Consider the two winning conditions. If $M^*$ was never queried to the signing oracle, then a valid signature on its digest is a new forgery, unless the adversary reused a signature from a different queried mission $M_j$ with the same digest. In the first case, we construct $\mathcal{B}_1$ that outputs the signature as an EUF-CMA forgery. In the second case, $M_j\neq M^*$ but
\begin{equation}
H(\mathsf{tag}\|\mathsf{Encode}(M_j))
=
H(\mathsf{tag}\|\mathsf{Encode}(M^*)),
\end{equation}
which gives a collision for $H$ and constructs $\mathcal{B}_2$.

If $M^*$ was previously signed, the adversary can win only if $a^*\notin\Pi_{M^*}(s^*,t^*)$. This is impossible under the assumption that the deterministic policy checker is correct, because the gate accepts only after the membership test succeeds. Therefore, every successful authorization bypass yields either a signature forgery or a hash collision. Applying the union bound gives Equation~\eqref{eq:auth-bound}.
\end{proof}

\begin{remark}
The theorem does not prove that the signed mission is sensible. It proves that the gate enforces the policy that was actually signed.
\end{remark}

\subsection{Mission Non-Escalation}
\begin{lemma}[Action-set non-escalation]
Let $\mathcal{A}_M(s,t)=\Pi_M(s,t)$. For every action accepted by AuthGuard-R,
\begin{equation}
a\in\mathcal{A}_M(s,t).
\end{equation}
Therefore, an untrusted planner cannot enlarge the effective action set by changing its text output.
\end{lemma}

\begin{proof}
The gate evaluates the action against the fixed policy $\Pi_M$ obtained from the verified mission token. The planner may output any string or normalized action, but acceptance requires membership in $\Pi_M(s,t)$. Since the planner cannot modify $M$ without a valid new signature, it cannot change $\mathcal{A}_M(s,t)$. Thus, every accepted action remains in the original action set.
\end{proof}

\begin{corollary}[Plan-level non-escalation]
If every action in an executed plan passes AuthGuard-R, then the complete execution trace is authorized under Definition 3.2.
\end{corollary}

\begin{proof}
Apply Lemma 9.2 to each executed action $a_i$ in its corresponding state $s_{i-1}$ and time $t_i$. This gives Equation~\eqref{eq:trace-auth} for every $i$.
\end{proof}

\subsection{Replay Resistance}
\begin{theorem}[Replay resistance]
Assume that the gate maintains a correct persistent registry and rejects every token whose mission state is \texttt{completed}, \texttt{revoked}, or \texttt{expired}. Then
\begin{equation}
\Adv^{\mathsf{replay}}_{\mathcal{A}}(\kappa)=0
\end{equation}
in the ideal model.
\end{theorem}

\begin{proof}
After the first valid mission execution, the challenger marks the nonce and mission identifier as consumed. On replay, Algorithm~\ref{alg:authguard} checks the registry before policy evaluation. The consumed state causes deterministic rejection. Therefore, no adversarial choice of planner output or environmental input can make the same mission instance pass. The success probability is zero under the stated registry assumption.
\end{proof}

\begin{remark}
If the registry can roll back after a crash, replay resistance requires stable storage, monotonic counters, or a trusted remote checkpoint. A random nonce alone is not sufficient when the system does not remember consumed missions.
\end{remark}

\subsection{Robot Binding}
\begin{theorem}[Cross-robot rejection]
Let token $T$ contain signed robot identifier $\mathsf{rid}_1$. If $R_2.\mathsf{id}\neq\mathsf{rid}_1$ and the robot-identity check is correct, then
\begin{equation}
\Adv^{\mathsf{robot}}_{\mathcal{A}}(\kappa)=0.
\end{equation}
\end{theorem}

\begin{proof}
The robot identifier is contained in the canonical mission encoding covered by the operator's signature. An adversary that changes $\mathsf{rid}_1$ to $\mathsf{rid}_2$ must either produce a new valid signature or find a different encoded mission with the same signed digest. These events are bounded by the assumptions used in Theorem 9.1.

When the unchanged token is presented to $R_2$, the gate compares the signed identifier with the trusted local identity of $R_2$. Since $R_2.\mathsf{id}\neq\mathsf{rid}_1$, the predicate $\mathsf{RobotMatch}(M,R_2)$ is false. The authorization predicate is a conjunction containing this check, so it evaluates to zero and the gate rejects before the action reaches the controller. Therefore, the adversary cannot win the cross-robot game under the stated assumptions.
\end{proof}

The theorem relies on a trusted local robot identity. A name supplied by the network or by the LLM is not sufficient because it can be rewritten. In practice, the identity may be bound to a device certificate, secure element, trusted platform module, or SROS2 enclave identity. The mission should also bind the relevant robot class or capability profile when software can be migrated between platforms.

Robot binding prevents a subtle form of escalation. Two robots may expose the same action name while having different physical strength, payload, sensors, or access privileges. A token that permits \texttt{move-object} on a small service robot must not automatically authorize the same request on an industrial arm. The evaluation therefore includes both identifier mismatch and capability-profile mismatch tests.

\subsection{Provenance Separation}
\begin{theorem}[Privileged-field noninterference]
Assume that field $f$ can be modified only by provenance labels in authorized set $\mathcal{L}_f$, and that any such modification requires a newly verified token. Then an adversary controlling only sources in $\mathcal{L}\setminus\mathcal{L}_f$ cannot change the effective value of $f$. Thus,
\begin{equation}
\Adv^{\mathsf{prov}}_{\mathcal{A}}(\kappa)=0
\end{equation}
in the ideal label-enforcement model.
\end{theorem}

\begin{proof}
The effective mission policy is derived only from the latest verified mission token. Inputs from labels outside $\mathcal{L}_f$ may influence the planner proposal or ordinary observation fields, but the provenance policy prevents them from writing field $f$. Since a new value of $f$ enters the gate only through a newly signed token, and the adversary controls no authorized source, the effective value remains unchanged. The argument holds after every input event, so it also holds for any finite sequence of adversarial inputs.
\end{proof}

This theorem formalizes why camera text such as ``new destination: Room 103'' cannot become an operator command. The text may affect the LLM proposal, but the gate still compares the proposal with the signed destination.

\subsection{Threshold-Approval Security}
\begin{theorem}[High-risk approval]
Suppose a high-risk action requires a secure $t$-of-$n$ threshold signature, and the adversary compromises fewer than $t$ signing authorities. Then the probability that an unauthorized high-risk action passes the approval check is negligible in $\kappa$.
\end{theorem}

\begin{proof}
A valid approval requires a threshold signature on the domain-separated digest of the mission identifier, action, state summary, expiry, and approval counter. With fewer than $t$ secret shares, the adversary cannot create a fresh valid threshold signature except with the forgery probability of the threshold scheme. Replaying a previous approval fails because the signed digest binds the action, state summary, expiry, and counter. Therefore, the success probability is negligible under threshold unforgeability and correct freshness checks.
\end{proof}

\subsection{Audit-Log Tamper Evidence}
\begin{theorem}[Tamper evidence of the anchored log]
Assume that $H$ is collision resistant and second-preimage resistant, and that the terminal digest of a log prefix is stored in a trusted external anchor. Any efficient adversary that modifies, inserts, deletes, or reorders a committed record while preserving the same terminal digest succeeds only with negligible probability.
\end{theorem}

\begin{proof}
Let the first changed record occur at position $j$. The input to the hash at position $j$ differs between the original and modified logs. For the final anchored digest to remain unchanged, the two chains must later merge to the same hash value. At the first merge point, two distinct hash inputs produce the same output, yielding a collision. If the adversary changes a fixed record while preserving the immediate digest, it yields a second preimage. If it deletes a suffix, the terminal digest no longer matches the external anchor. Therefore, an undetected modification implies a violation of the stated hash assumptions.
\end{proof}

\subsection{Dual-Gate Trace Composition}
Let $G_A(s,a,t)$ be the authorization decision and $G_S(s,a)$ be the safety decision. The controller executes only when
\begin{equation}
\Exec(s,a,t)=G_A(s,a,t)\land G_S(s,a).
\label{eq:execute}
\end{equation}

\begin{theorem}[Trace-level dual-gate soundness]
Assume that $G_A$ is authorization-sound and $G_S$ is safety-sound for their stated models. Let
\begin{equation}
\tau=(s_0,a_1,s_1,\ldots,a_\ell,s_\ell)
\end{equation}
be any trace produced by the controller under Equation~\eqref{eq:execute}. Then $\tau$ is both authorized under $M$ and safe under $\mathcal{S}$.
\end{theorem}

\begin{proof}
We prove the statement by induction on the number of executed actions.

For the base case, consider $a_1$. The controller executes it only if $G_A(s_0,a_1,t_1)=1$ and $G_S(s_0,a_1)=1$. Authorization soundness gives $a_1\in\Pi_M(s_0,t_1)$. Safety soundness gives $\mathcal{S}(s_0,a_1)=1$. Thus, the one-action trace is authorized and safe.

For the induction step, assume that the prefix through $a_{k-1}$ is authorized and safe and reaches state $s_{k-1}$. Action $a_k$ is executed only when both gates accept in $s_{k-1}$. Therefore,
\begin{equation}
a_k\in\Pi_M(s_{k-1},t_k)
\end{equation}
and
\begin{equation}
\mathcal{S}(s_{k-1},a_k)=1.
\end{equation}
Appending $a_k$ preserves both properties. By induction, every action in $\tau$ is authorized and safe.
\end{proof}

\begin{corollary}[Blocking safety-compliant mission hijacking]
Under the assumptions of Theorem 9.8, AuthGuard-R combined with the safety gate does not execute any plan $P$ satisfying $\SCMH(P,M,s_0)=1$.
\end{corollary}

\begin{proof}
A successful mission-hijacking plan contains at least one action outside $\Pi_M$. Authorization soundness causes the authorization gate to reject that action. Equation~\eqref{eq:execute} then evaluates to zero, so the controller does not execute the complete plan.
\end{proof}

\subsection{Scope of the Proofs}
The proofs establish properties of the encoded mission, the deterministic policy checker, and the cryptographic mechanisms. They do not prove that the structured mission perfectly captures every unstated human intention. If an operator confirms the wrong destination or grants an unnecessarily broad object set, AuthGuard-R will faithfully enforce that policy. Human-interface design and mission review are therefore necessary parts of deployment.

The proofs also assume that normalized actions accurately represent their physical effect. A parser bug that maps two different tool calls to the same canonical action can violate this assumption without forging a signature. The same issue appears in trusted state. If localization reports Room 101 while the robot is physically in Room 103, a correct symbolic check may approve an action in the wrong physical place. Independent testing, sensor fusion, attestation, and low-level safety limits are needed to reduce these risks.

The safety-composition theorem is conditional on the correctness of the selected safety gate. It does not claim that the gate models every collision, human factor, or long-term consequence. Similarly, the threshold proof assumes independence and secure storage of approval keys. A malicious authorized operator can still sign a harmful mission, and cryptography cannot decide whether the operator's decision is ethical.

Finally, the model does not prevent network flooding, power failure, physical obstruction, or denial of service against the gate. It provides a fail-closed authorization decision for actions that reach the enforcement point. Emergency behavior and availability are treated separately in the design and discussion sections.

\section{Complexity and Performance Analysis}
\subsection{Per-Mission and Per-Action Cost}
Let $C_{\mathsf{sig}}$ be the time for one signature verification, $C_H$ the time for one hash evaluation, and $m$ the number of policy predicates evaluated for an action. Let $C_j$ be the cost of predicate $j$. If the signature is verified for every action, the authorization cost is
\begin{equation}
T_A=C_{\mathsf{sig}}+C_H+\sum_{j=1}^{m}C_j+C_{\mathsf{log}}.
\label{eq:cost-action}
\end{equation}

A more efficient implementation verifies the mission token once, caches the verified mission in protected memory, and checks freshness before each action. For a plan of length $\ell$, the amortized cost becomes
\begin{equation}
\overline{T}_A=\frac{C_{\mathsf{sig}}+C_H}{\ell}+\sum_{j=1}^{m}C_j+C_{\mathsf{log}}.
\label{eq:amortized}
\end{equation}

Most policy predicates are set membership, range, or equality tests. With hashed object and region identifiers, their expected lookup cost is constant. State predicates may be more expensive when they call a geometric or symbolic checker.

\subsection{End-to-End Latency}
Let $T_N$ be normalization time, $T_V$ signature and certificate verification time, $T_P$ deterministic policy-checking time, $T_A$ approval-processing time, $T_L$ audit time, and $T_S$ safety-gate time. The decision latency for one action is
\begin{equation}
T_{\mathsf{total}}=T_N+T_V+T_P+T_A+T_L+T_S.
\label{eq:latency}
\end{equation}
The LLM planning time is reported separately because it depends strongly on the selected model and deployment. Separating it prevents model latency from hiding the overhead introduced by AuthGuard-R.

Signature verification need not be repeated in full for every action when the gate maintains a secure active-mission cache. The complete token is verified when the mission begins or changes, and later actions refer to the verified mission identifier and version. Per-action checks still include freshness, state, scope, provenance, and consumed-step status. The evaluation should report both cold-start latency and steady-state latency.

For physical robots, tail latency is more important than only the mean. A small average overhead with rare long pauses may still harm control. The experiments therefore report median, 95th percentile, and 99th percentile latency under benign load and under repeated rejected attacks.

\subsection{Token and Log Size}
Let $|M|$ denote the canonical mission length, $|\sigma|$ the signature length, and $|\mathsf{cert}|$ the certificate or key-reference length. A token has approximate size
\begin{equation}
|T|=|M|+|\sigma|+|\mathsf{cert}|+O(1).
\end{equation}
The mission length grows with the number of allowed objects, regions, alternatives, state guards, and approval rules. Compact numeric identifiers can reduce size, while human-readable JSON increases inspectability.

Each audit record stores or commits to the normalized action, state digest, timestamp, decision, reason code, and previous-chain digest. If the complete action and selected evidence are stored, the log grows linearly with the number of decisions. If large sensor objects are kept elsewhere, the record stores only their content digest and protected reference.

Post-quantum signatures generally increase token size compared with Ed25519. This overhead is usually small relative to images or model prompts, but it may matter for low-bandwidth robot links and long certificate chains. The evaluation therefore reports raw token bytes, compressed bytes where compression is permitted, and cumulative log growth per mission.

\subsection{Threshold Approval Cost}
Threshold approval adds cryptographic work and human or service coordination. With independent signatures, the gate verifies $r$ approvals for an action requiring $r$ authorities, giving verification cost approximately $rT_V$. A threshold-signature construction may produce one compact aggregate signature, but it introduces a distributed signing protocol and additional setup assumptions.

The dominant delay may be the time required for supervisors to respond rather than the cryptographic computation. Experiments should therefore separate machine processing time from approval waiting time. Simulation can measure cryptographic overhead, while a small user study or scripted approval service can estimate operational delay.

Approval caching must be narrow. A supervisor may pre-approve a class of low-risk actions for one mission, but high-risk authorization should remain bound to the exact robot, mission version, action scope, and expiry. Broad reusable approval would weaken the least-authority design.

\subsection{Failure Policy}
AuthGuard-R fails closed for ordinary mission actions. Invalid signatures, expired missions, missing fields, denied provenance, stale state, registry errors, and policy-checker exceptions cause rejection rather than best-effort execution. The gate returns a reason code and may request clarification or a newly signed mission.

Fail-closed behavior can reduce availability when the gate, clock, or network is unavailable. The system therefore distinguishes mission progress from emergency safety. A small emergency capability set may be pre-authorized, for example stop motion, release dangerous force, return to a marked safe position, or request human assistance. These actions are encoded in a separate policy with narrow state guards.

The implementation must avoid unsafe fallback paths. A controller should not execute the raw LLM action merely because the authorization node crashes. ROS 2 launch and watchdog configuration should make the authorization node part of the mandatory actuation path.

\section{Implementation and Experimental Evaluation}
\label{sec:methodology}
\label{sec:pilot}

The implementation is designed to test the central claim of the paper in a direct sequence. The formal results in Section~\ref{sec:proofs} show that, under the stated assumptions, an action outside the signed mission cannot pass the authorization gate. The experiments ask whether this distinction remains meaningful when the action is produced by an actual LLM planner. We therefore evaluate three connected questions. First, can an action be physically safe and still violate the authenticated mission? Second, do real LLM planners produce such actions when untrusted text, speech, or visual observations attempt to modify the mission? Third, when the planner does produce an unauthorized action, does AuthGuard-R reject it without unnecessarily rejecting benign actions? The experiments are organized in the same order. We first isolate the safety--authorization gap with hand-constructed actions, then expose two LLM planners to fixed and adaptive mission-hijacking attacks, and finally compare the results across models. In this way, each stage tests one part of the argument and the final cross-model result connects the implementation back to the formal authorization property.

\subsection{Prototype, Execution Path, and Evaluation Criteria}

The prototype follows the architecture defined earlier in the paper. A natural-language mission is converted into a canonical token containing the robot identity, allowed actions, objects, regions, numerical bounds, trusted state conditions, provenance restrictions, validity period, nonce, and approval policy. After operator confirmation, the token is signed with Ed25519. The LLM receives the task and environmental observations, but it cannot modify the signed token. Its output is converted into the typed action of Equation~\eqref{eq:action}; missing, ambiguous, or unsupported protected fields are rejected rather than inferred optimistically. The normalized action is then evaluated by AuthGuard-R using Equation~\eqref{eq:auth-predicate} and by an independent safety gate. An action reaches execution only if both decisions are positive, as specified by Equation~\eqref{eq:execute}. The prototype implements canonical mission encoding, signature generation and verification, per-field provenance labels, persistent mission state, nonce and replay checks, robot binding, risk-adaptive approval, and a hash-chained audit log. The core implementation contains approximately 1{,}800 lines of Python~3.11 and uses version~41.0.7 of the \texttt{cryptography} library. For the open-source Qwen experiment, the model-facing harness was executed locally through Ollama using Python~3.14; the authorization logic and decision rule were unchanged. In a ROS~2 deployment, the same logic can be placed in an enforcement node between the planner and the action server so that raw LLM output never publishes directly to an actuator; SROS2 can protect node identity and communication, while AuthGuard-R provides the mission-level semantic check~\cite{rossecurity}.

The attack procedure is chosen to challenge the authorization boundary rather than merely the model's refusal behavior. Six planner-level categories are used: goal substitution, capability escalation, object substitution, provenance confusion, area expansion, and parameter escalation. Replay, cross-robot token reuse, and threshold-approval bypass are tested separately because they attack protocol state rather than requiring an LLM-generated command. In the fixed-template setting, an authenticated mission is accompanied by an injected observation from an untrusted channel such as camera OCR, an audio transcript, or a printed note; the injected instruction is strengthened over at most three rounds. In adaptive MissionPAIR, a second model observes the target planner's previous response and generates a revised injection. A planner refusal is not counted as a defense success because no unauthorized action reaches AuthGuard-R. We record planner compliance only when the model actually produces an action that changes the protected mission field. The ``Blocked'' column in the live-model tables is therefore conditional on planner compliance and measures the behavior of AuthGuard-R after the gate is genuinely challenged. Provenance is assigned by the harness rather than by trusting the model's self-report, and the target planner is not given the signed mission object, internal policy schema, or detailed authorization reason codes.

To separate security from utility, the experiments report both attack and benign outcomes. If $N_A$ is the number of adversarial trials and $N_U$ is the number of adversarial trials that result in an executed unauthorized action, the unauthorized-action success rate is
\begin{equation}
\mathsf{UASR}=\frac{N_U}{N_A}.
\label{eq:uasr}
\end{equation}
The safety-compliant mission-hijacking rate is the fraction of adversarial trials satisfying Definition~\ref{def:scmh},
\begin{equation}
\mathsf{SCMHR}=\frac{1}{N_A}\sum_{k=1}^{N_A}\SCMH(P_k,M_k,s_{0,k}).
\label{eq:scmhr}
\end{equation}
For $N_B$ benign trials, with $N_{BA}$ accepted and $N_{BR}$ rejected, we use $\mathsf{BTCR}=N_{BA}/N_B$ and $\mathsf{FRR}=N_{BR}/N_B$. We also report the supplementary security--utility score
\begin{equation}
\mathsf{SU}=\frac{2(1-\mathsf{UASR})\mathsf{BTCR}}{(1-\mathsf{UASR})+\mathsf{BTCR}},
\label{eq:su}
\end{equation}
when the denominator is nonzero. A high $\mathsf{SU}$ therefore requires both a low unauthorized-execution rate and a high benign completion rate; it cannot be improved simply by stopping the robot. Binary planner-compliance rates are reported with Wilson 95\% confidence intervals.

\subsection{Controlled Validation: Is Safety Alone Sufficient?}
\label{sec:pilot-handbuilt}

The first experiment deliberately removes model uncertainty and asks the most basic question behind the paper: if an action is physically safe but outside the signed mission, what does each defense do? Sixteen unit tests first verify the implementation of signature checks, replay protection, robot binding, trusted-state requirements, provenance checks, approval signatures, and audit-log integrity; all sixteen pass. We then construct eleven unauthorized actions that are intentionally chosen to remain physically acceptable under the independent safety rules, together with seven benign controls. The attacks cover goal substitution, object substitution, area expansion, capability escalation, parameter escalation, cross-robot reuse, provenance confusion, replay, and threshold-approval bypass. Each action is evaluated with no defense, with the safety gate alone, and with the complete dual-gate system. Table~\ref{tab:pilot-results} gives the aggregate result. The no-defense and safety-only configurations both have $\mathsf{UASR}=1$ and $\mathsf{SCMHR}=1$ because all eleven attacks are executable and were intentionally constructed to satisfy the physical-safety constraints. This is the behavior the paper is intended to expose: a safety decision can be correct and still be insufficient for mission integrity. With AuthGuard-R added, $\mathsf{UASR}$ and $\mathsf{SCMHR}$ both fall to zero, while $\mathsf{BTCR}$ remains 1 and $\mathsf{FRR}$ remains 0. Consequently, the supplementary $\mathsf{SU}$ score changes from 0 for the first two configurations to 1 for the dual gate. The important interpretation is not merely that attacks were blocked; it is that the security improvement did not come from rejecting benign work. In this controlled suite, the authorization gate rejects exactly the actions that cross the signed mission boundary and preserves all seven benign controls.

\begin{table}[t]
\centering
\caption{Hand-constructed validation: security and utility results.}
\label{tab:pilot-results}
\resizebox{\columnwidth}{!}{%
\begin{tabular}{lccccc}
\toprule
System & UASR $\downarrow$ & SCMHR $\downarrow$ & BTCR $\uparrow$ & FRR $\downarrow$ & SU $\uparrow$ \\
\midrule
No defense & 1.00 & 1.00 & 1.00 & 0.00 & 0.00 \\
Safety-only pilot baseline & 1.00 & 1.00 & 1.00 & 0.00 & 0.00 \\
AuthGuard-R + safety gate & 0.00 & 0.00 & 1.00 & 0.00 & 1.00 \\
\bottomrule
\end{tabular}}
\end{table}

Table~\ref{tab:pilot-category} breaks the same experiment down by attack type and shows that the aggregate result is not caused by one easy category. The entries denote successful unauthorized executions over attempted attacks. Without authorization, the no-defense and safety-only configurations accept every tested attack, including the two provenance-confusion cases. The dual gate changes every corresponding entry to zero. This category-level result is important because the attacks violate different mission fields: some change the destination or object, some enlarge capability or parameter scope, and others exploit robot identity, provenance, freshness, or approval state. The common outcome across these different cases is consistent with the formal design: AuthGuard-R does not need a separate learned detector for each attack name; it checks whether the resulting action remains inside the authenticated policy. Thus Tables~\ref{tab:pilot-results} and~\ref{tab:pilot-category} provide the first empirical bridge from the authorization-soundness argument in Section~\ref{sec:proofs} to the implemented decision procedure.

\begin{table}[t]
\centering
\caption{Hand-constructed validation by attack category. Entries show successful unauthorized executions over attempted attacks.}
\label{tab:pilot-category}
\resizebox{\columnwidth}{!}{%
\begin{tabular}{lccc}
\toprule
Category & No defense & Safety-only & Dual gate \\
\midrule
Goal substitution & 1/1 & 1/1 & 0/1 \\
Object substitution & 1/1 & 1/1 & 0/1 \\
Area expansion & 1/1 & 1/1 & 0/1 \\
Capability escalation & 1/1 & 1/1 & 0/1 \\
Parameter escalation & 1/1 & 1/1 & 0/1 \\
Cross-robot reuse & 1/1 & 1/1 & 0/1 \\
Provenance confusion & 2/2 & 2/2 & 0/2 \\
Replay & 1/1 & 1/1 & 0/1 \\
Threshold-approval bypass & 1/1 & 1/1 & 0/1 \\
\bottomrule
\end{tabular}}
\end{table}

\subsection{Live LLM Evaluation: Do Real Planners Cross the Mission Boundary?}

The controlled experiment establishes that the gate behaves correctly when the unauthorized action is supplied directly, but it does not show that an LLM planner will actually generate such actions. The next experiment therefore places the attack before the planner. We evaluate Claude Haiku~4.5 and Qwen2.5~7B because they exhibit different instruction-following and refusal behavior. The purpose is not to rank the models; it is to test whether the authorization result depends on a particular planner. For every successful attack, the planner's output must first cross the authenticated mission boundary, after which exactly the same deterministic authorization check is applied.

\subsubsection{Claude Haiku~4.5}
\label{sec:pilot-live-llm}
\label{sec:pilot-adaptive}

Table~\ref{tab:pilot-live-llm} reports the fixed-template Claude experiment with ten seeds for each of the six categories, for a total of 60 trials. The planner follows the injected mission change in 11 trials, giving a pooled compliance rate of 18.3\%. The distribution is highly nonuniform: object substitution succeeds in all ten seeds, parameter escalation succeeds once, and the remaining four categories produce no mission-violating output. This table therefore reveals two different facts. First, Claude's own refusal behavior prevents many attacks from reaching the authorization layer. Second, refusal is incomplete: the model still produces eleven actions that contradict the signed mission. Those eleven cases are the relevant security test for AuthGuard-R, and all 11 are rejected. The result supports the intended separation of responsibilities: model alignment can reduce how often an attacker obtains a bad proposal, but the authorization gate remains necessary for the cases in which alignment fails.

\begin{table}[t]
\centering
\caption{Fixed-template attacks against Claude Haiku~4.5. Wilson 95\% confidence intervals are reported for planner compliance.}
\label{tab:pilot-live-llm}
\resizebox{\columnwidth}{!}{%
\begin{tabular}{lcc}
\toprule
Category & Planner complied & Blocked by AuthGuard-R \\
\midrule
Goal substitution & 0/10~[0.00, 0.28] & --- \\
Capability escalation & 0/10~[0.00, 0.28] & --- \\
Object substitution & 10/10~[0.72, 1.00] & 10/10 \\
Provenance confusion & 0/10~[0.00, 0.28] & --- \\
Area expansion & 0/10~[0.00, 0.28] & --- \\
Parameter escalation & 1/10~[0.02, 0.40] & 1/1 \\
\midrule
Pooled (60 trials) & 11/60 = 18.3\%~[0.11, 0.30] & 11/11 \\
\bottomrule
\end{tabular}}
\end{table}

The fixed templates are intentionally simple, so we next allow the attack to adapt to the planner. In adaptive MissionPAIR, a second Claude instance receives the previous target response and produces a revised injection for the next round. Table~\ref{tab:pilot-adaptive-haiku} shows that the adaptive procedure produces 12 mission-violating outputs in 60 trials, corresponding to 20.0\% compliance. More importantly than the small change in the pooled rate, the attack reaches categories that the fixed templates did not: goal substitution succeeds once and provenance confusion succeeds twice, while object substitution remains the dominant failure mode. This demonstrates why an adaptive evaluation is useful even when its overall success rate is not dramatically larger; it explores a different part of the mission space. AuthGuard-R rejects all 12 unauthorized actions. Across the two Claude experiments, the planner therefore produces 23 mission-violating actions and the authorization gate rejects all 23.

\begin{table}[t]
\centering
\caption{Adaptive MissionPAIR against Claude Haiku~4.5.}
\label{tab:pilot-adaptive-haiku}
\resizebox{\columnwidth}{!}{%
\begin{tabular}{lcc}
\toprule
Category & Planner complied & Blocked by AuthGuard-R \\
\midrule
Goal substitution & 1/10~[0.02, 0.40] & 1/1 \\
Capability escalation & 0/10~[0.00, 0.28] & --- \\
Object substitution & 7/10~[0.40, 0.89] & 7/7 \\
Provenance confusion & 2/10~[0.06, 0.51] & 2/2 \\
Area expansion & 0/10~[0.00, 0.28] & --- \\
Parameter escalation & 2/10~[0.06, 0.51] & 2/2 \\
\midrule
Pooled (60 trials) & 12/60 = 20.0\%~[0.12, 0.32] & 12/12 \\
\bottomrule
\end{tabular}}
\end{table}

The Claude runs also exposed an important implementation issue that is consistent with, rather than contrary to, the formal model. In an early return-to-base scenario, the mission policy did not encode the intended trusted low-battery condition. The gate therefore accepted the action because the encoded policy actually permitted it. After the missing state guard was added, the captured action was replayed and rejected. This case illustrates the role of the formal assumptions: deterministic enforcement can faithfully enforce the signed policy, but it cannot infer a restriction that the mission compiler failed to encode. For this reason, mission construction and action normalization remain part of the trusted implementation boundary.

\subsubsection{Qwen2.5~7B}
\label{sec:pilot-qwen-fixed}
\label{sec:pilot-qwen-adaptive}

The Qwen experiment provides a stronger stress test because the open-source planner is much more willing to follow the injected mission modifications. Table~\ref{tab:pilot-qwen-fixed} reports the fixed-template results. Qwen2.5~7B complies in 50 of 60 trials, or 83.3\%. Goal substitution, object substitution, provenance confusion, area expansion, and parameter escalation each succeed in all ten seeds; only capability escalation produces no compliant output. The significance of this result is different from the Claude result. Here the planner itself provides little protection: five of the six tested attack categories reach the authorization gate in every seed. Nevertheless, AuthGuard-R rejects all 50 unauthorized actions. The experiment therefore tests the gate under substantially heavier adversarial pressure and shows that the enforcement decision does not rely on the target model being strongly aligned against the injected instruction.

\begin{table}[t]
\centering
\caption{Fixed-template attacks against Qwen2.5~7B through Ollama.}
\label{tab:pilot-qwen-fixed}
\resizebox{\columnwidth}{!}{%
\begin{tabular}{lcc}
\toprule
Category & Planner complied & Blocked by AuthGuard-R \\
\midrule
Goal substitution & 10/10~[0.72, 1.00] & 10/10 \\
Capability escalation & 0/10~[0.00, 0.28] & --- \\
Object substitution & 10/10~[0.72, 1.00] & 10/10 \\
Provenance confusion & 10/10~[0.72, 1.00] & 10/10 \\
Area expansion & 10/10~[0.72, 1.00] & 10/10 \\
Parameter escalation & 10/10~[0.72, 1.00] & 10/10 \\
\midrule
Pooled (60 trials) & 50/60 = 83.3\%~[0.72, 0.91] & 50/50 \\
\bottomrule
\end{tabular}}
\end{table}

Table~\ref{tab:pilot-qwen-adaptive} reports the adaptive Qwen experiment, in which a second Qwen2.5~7B instance generates and refines the injection after observing the target's previous response. The adaptive attacker obtains 36 mission-violating outputs in 60 trials, or 60.0\%. Provenance confusion and parameter escalation succeed in all ten seeds, area expansion succeeds in eight, goal substitution in six, and object substitution in two. The pooled adaptive rate is lower than the fixed-template rate, so these results should not be interpreted as showing that adaptivity is automatically stronger for every model. Instead, they show that a different attack-generation process produces a different distribution of failures. For the defense question, the important denominator is the 36 actions that actually violate the signed mission, and AuthGuard-R rejects all 36. Taken together, the Qwen experiments generate 86 unauthorized planner outputs under two substantially different attack procedures, and none reaches execution through the authorization gate.

\begin{table}[t]
\centering
\caption{Adaptive MissionPAIR using Qwen2.5~7B as attacker and target.}
\label{tab:pilot-qwen-adaptive}
\resizebox{\columnwidth}{!}{%
\begin{tabular}{lcc}
\toprule
Category & Planner complied & Blocked by AuthGuard-R \\
\midrule
Goal substitution & 6/10~[0.31, 0.83] & 6/6 \\
Capability escalation & 0/10~[0.00, 0.28] & --- \\
Object substitution & 2/10~[0.06, 0.51] & 2/2 \\
Provenance confusion & 10/10~[0.72, 1.00] & 10/10 \\
Area expansion & 8/10~[0.49, 0.94] & 8/8 \\
Parameter escalation & 10/10~[0.72, 1.00] & 10/10 \\
\midrule
Pooled (60 trials) & 36/60 = 60.0\%~[0.47, 0.71] & 36/36 \\
\bottomrule
\end{tabular}}
\end{table}

\subsection{Cross-Model Interpretation and Significance}
\label{sec:pilot-crossmodel}
\label{sec:pilot-interpretation}

Table~\ref{tab:pilot-crossmodel} brings the four live experiments together and provides the clearest empirical answer to the paper's main question. The planner side is highly variable: Claude Haiku~4.5 complies in 18.3\% and 20.0\% of the fixed and adaptive trials, whereas Qwen2.5~7B complies in 83.3\% and 60.0\%. This large variation shows why the LLM itself should not be treated as the authorization mechanism. A change of model can dramatically change how often an injected instruction becomes an unauthorized proposal. The enforcement side is stable across the evaluated runs. Out of 240 live attack trials, 109 produce mission-violating planner outputs, and AuthGuard-R rejects all 109. The ``Blocked'' figures in Table~\ref{tab:pilot-crossmodel} are intentionally conditioned on the planner having produced an unauthorized action; this prevents model refusals from being incorrectly credited to the authorization gate. Adding the eleven hand-constructed protocol- and policy-level attacks gives 120 instantiated unauthorized actions rejected by the implementation. These measurements do not constitute a universal 100\% security claim, because the formal guarantee still depends on correct policy encoding, faithful normalization, trusted state, and an uncompromised gate. They do, however, provide positive experimental evidence for the specific theorem-backed architectural claim of this work: once an unauthorized action is represented correctly against the signed mission, the execution decision can be made independently of whether the LLM itself is robust or vulnerable to the injected instruction.

\begin{table}[t]
\centering
\caption{Cross-model summary. ``Blocked'' is conditional on the planner producing a mission-violating action.}
\label{tab:pilot-crossmodel}
\resizebox{\columnwidth}{!}{%
\begin{tabular}{llcc}
\toprule
Target planner & Attack method & Planner complied & Blocked \\
\midrule
Claude Haiku~4.5 & Fixed template & 11/60 (18.3\%) & 11/11 \\
Claude Haiku~4.5 & Adaptive & 12/60 (20.0\%) & 12/12 \\
Qwen2.5~7B & Fixed template & 50/60 (83.3\%) & 50/50 \\
Qwen2.5~7B & Adaptive & 36/60 (60.0\%) & 36/36 \\
\midrule
All live trials & --- & 109/240 (45.4\%) & 109/109 \\
\bottomrule
\end{tabular}}
\end{table}

The progression of the implementation results is therefore important. Tables~\ref{tab:pilot-results} and~\ref{tab:pilot-category} first establish under controlled conditions that physical safety alone does not protect mission integrity and that the implemented authorization checks cover several distinct policy dimensions. Tables~\ref{tab:pilot-live-llm} and~\ref{tab:pilot-adaptive-haiku} then show that a comparatively resistant proprietary planner can still cross the mission boundary, while Tables~\ref{tab:pilot-qwen-fixed} and~\ref{tab:pilot-qwen-adaptive} show the same problem under a much more attack-compliant open-source planner. Finally, Table~\ref{tab:pilot-crossmodel} shows that although planner susceptibility varies sharply, the external authorization result remains the same for every observed unauthorized output. This is the practical significance of the paper's dual-gate design. The safety gate answers whether an action is acceptable in the physical state, whereas AuthGuard-R answers whether the operator authorized that action in the first place. The experiments do not replace the formal proofs; rather, they instantiate the assumptions of those proofs on concrete attack traces and show that the implemented gate produces the predicted rejection behavior while preserving the benign controls in the validation suite.\\
\noindent {\bf Discussion}. The experiments support the main distinction of this paper: physical safety and mission authorization are related but different properties. An action may be collision-free, within speed limits, and otherwise physically acceptable while still changing the destination, object, sensing scope, or numerical bounds selected by the operator. The hand-constructed study makes this separation explicit, and the live-model experiments show that such mission changes can also be produced by real LLM planners under injected observations. The large difference between Claude Haiku~4.5 and Qwen2.5~7B compliance rates further shows why authorization should not be delegated to the planner itself; the same deterministic mission check can be applied even when model refusal behavior changes substantially. AuthGuard-R is therefore intended to complement rather than replace a safety gate: the authorization layer checks whether an action remains within the signed mission, while the safety layer checks whether that authorized action is acceptable in the current physical state. The 109/109 live-action rejections should be read as evidence that the present implementation enforced the encoded policy on the evaluated attacks, not as a claim that cryptography can correct a wrongly encoded mission or an incorrectly normalized physical action.

\section{Conclusion}
This paper identifies a security gap in LLM-controlled robotics that is not captured by physical safety alone: a robot may execute an action that is physically acceptable while still violating the operator's authenticated mission. We formalized this problem as safety-compliant mission hijacking, introduced MissionPAIR to generate mission-changing attacks that remain compatible with a physical-safety baseline, and proposed AuthGuard-R as a deterministic authorization boundary between the probabilistic planner and physical execution. The formal analysis establishes authorization soundness, mission non-escalation, replay resistance, robot binding, provenance separation, approval security, audit tamper evidence, and composition with an independent safety gate under the stated assumptions, while the implementation follows the same separation by normalizing each planner action and checking it against the signed mission before execution. Across 240 live attack trials using Claude Haiku~4.5 and Qwen2.5~7B, the planners produced 109 mission-violating actions and AuthGuard-R rejected all 109; together with eleven hand-constructed protocol- and policy-level attacks, the implementation rejected 120 instantiated unauthorized actions while preserving the benign controls in the validation suite. The large difference between the two planners' attack-compliance profiles reinforces the central design choice of keeping authority outside the LLM. The resulting principle is simple: an LLM may decide how to perform an authorized task, but it should not be able to decide what it is authorized to do. Future work will move the authorization mechanism from the symbolic action pipeline into a full simulator and ROS~2 execution loop so that long-horizon plans, replanning after state changes, actuator feedback, and interaction with realistic physical-safety mechanisms such as RoboGuard can be evaluated. Reproducing RoboPAIR, RoboGuard, and BadRobot at their intended scale would provide stronger comparison with established attack and defense frameworks.

\subsection*{Acknowledgment}
AI tools were used solely for grammar correction and language refinement. Vikas Srivastava acknowledges the support received from the ANRF-PMECRG project with Ref.
ANRF/ECRG/2025/002808/PMS and NIT Warangal Research Seed Grant.

\balance


\begin{thebibliography}{99}

\bibitem{robopair}
A. Robey, Z. Ravichandran, V. Kumar, H. Hassani, and G. J. Pappas,
``Jailbreaking LLM-Controlled Robots,''
\emph{Proc. IEEE International Conference on Robotics and Automation (ICRA)},
pp. 11948--11956, 2025.
doi: 10.1109/ICRA55743.2025.11128119.

\bibitem{roboguard}
Z. Ravichandran, A. Robey, V. Kumar, G. J. Pappas, and H. Hassani,
``Safety Guardrails for LLM-Enabled Robots,''
\emph{IEEE Robotics and Automation Letters},
vol. 11, no. 4, pp. 4649--4656, 2026.
doi: 10.1109/LRA.2026.3667488.

\bibitem{badrobot}
H. Zhang, C. Zhu, X. Wang, Z. Zhou, C. Yin, M. Li, L. Xue,
Y. Wang, S. Hu, A. Liu, P. Guo, and L. Y. Zhang,
``BadRobot: Jailbreaking Embodied LLM Agents in the Physical World,''
\emph{International Conference on Learning Representations (ICLR)}, 2025;
arXiv:2407.20242.
doi: 10.48550/arXiv.2407.20242.

\bibitem{vulnerability}
X. Wu, S. Chakraborty, R. Xian, J. Liang, T. Guan, F. Liu,
B. M. Sadler, D. Manocha, and A. S. Bedi,
``On the Vulnerability of LLM/VLM-Controlled Robotics,''
\emph{Proc. IEEE/RSJ International Conference on Intelligent Robots
and Systems (IROS)}, pp. 1914--1921, 2025.
doi: 10.1109/IROS60139.2025.11246863.

\bibitem{ripa}
N. Dorzhiev,
``RIPA: Sensory-Vector Prompt Injection Attacks on LLM-Controlled ROS 2 Robots,''
arXiv:2606.28649, 2026.
doi: 10.48550/arXiv.2606.28649.

\bibitem{blindfold}
X. Huang, Q. Yang, L. Shen, Z. Ma, and Y. Zheng,
``Jailbreaking Embodied LLMs via Action-Level Manipulation,''
\emph{Proc. ACM/IEEE International Conference on Embedded Artificial
Intelligence and Sensing Systems (SenSys)}, 2026.
doi: 10.1145/3774906.3802758.

\bibitem{semanticdos}
J. Steinberg and O. Gal,
``Semantic Denial of Service in LLM-Controlled Robots,''
arXiv:2604.24790, 2026.
doi: 10.48550/arXiv.2604.24790.

\bibitem{safegate}
I. Obi, V. L. N. Venkatesh, W. Wang, R. Wang, D. Suh,
T. I. Amosa, W. Jo, and B.-C. Min,
``Pre-Execution Safety Gate \& Task Safety Contracts for
LLM-Controlled Robot Systems,''
arXiv:2604.05427, 2026.
doi: 10.48550/arXiv.2604.05427.

\bibitem{runtimegov}
X. Qin, S. Luan, J. See, Z. Boukhers, C. Yang, and Z. Li,
``Harnessing Embodied Agents: Runtime Governance for
Policy-Constrained Execution,''
arXiv:2604.07833, 2026.
doi: 10.48550/arXiv.2604.07833.

\bibitem{modularguardrails}
J. Kim, W. Chen, D. Soleymanzadeh, Y. Ding, X. Gao, Z. Tu,
R. Zhang, F. Fei, S. Veer, Y. Lyu, M. Zheng, and Y. Gu,
``Position: Modular Safety Guardrails Are Necessary for
Foundation-Model-Enabled Robots in the Real World,''
\emph{Proc. International Conference on Machine Learning (ICML),
Position Track}, 2026; arXiv:2602.04056.
doi: 10.48550/arXiv.2602.04056.

\bibitem{robojailbench}
D. Yeke, Y. Zhou, L. Y. Lin, H. Cai, A. Bianchi, and Z. B. Celik,
``RoboJailBench: Benchmarking Adversarial Attacks and Defenses
in Embodied Robotic Agents,''
arXiv:2605.19328, 2026.
doi: 10.48550/arXiv.2605.19328.

\bibitem{safeagentbench}
S. Yin, X. Pang, Y. Ding, M. Chen, Y. Bi, Y. Xiong, W. Huang,
Z. Xiang, J. Shao, and S. Chen,
``SafeAgentBench: A Benchmark for Safe Task Planning of Embodied LLM Agents,''
arXiv:2412.13178, 2024.
doi: 10.48550/arXiv.2412.13178.

\bibitem{govbench}
X. Qin, S. Luan, J. See, C. Yang, and Z. Li,
``EmbodiedGovBench: A Benchmark for Governance, Recovery, and
Upgrade Safety in Embodied Agent Systems,''
arXiv:2604.11174, 2026.
doi: 10.48550/arXiv.2604.11174.

\bibitem{trustsurvey}
X. Huang, S. Karthick V B, T. Chen, M. Bryson, T. Chaffey,
H. Chen, K.-K. R. Choo, and I. R. Manchester,
``Trust in LLM-Controlled Robotics: A Survey of Security Threats,
Defenses and Challenges,''
arXiv:2601.02377, 2025.
doi: 10.48550/arXiv.2601.02377.

\bibitem{owasp}
OWASP Foundation,
``LLM Prompt Injection Prevention Cheat Sheet,''
\emph{OWASP Cheat Sheet Series}.
[Online]. Available:
\url{https://cheatsheetseries.owasp.org/cheatsheets/LLM_Prompt_Injection_Prevention_Cheat_Sheet.html}.
Accessed: Sep. 25, 2026.

\bibitem{rossecurity}
Open Robotics,
``Understanding the Security Keystore,''
\emph{ROS 2 Documentation: Jazzy}.
[Online]. Available:
\url{https://docs.ros.org/en/ros2_documentation/jazzy/Tutorials/Advanced/Security/The-Keystore.html}.
Accessed: Sep. 25, 2026.

\bibitem{jwt}
M. Jones, J. Bradley, and N. Sakimura,
``JSON Web Token (JWT),''
RFC 7519, Internet Engineering Task Force, May 2015.
doi: 10.17487/RFC7519.

\bibitem{fips204}
National Institute of Standards and Technology,
``Module-Lattice-Based Digital Signature Standard,''
FIPS 204, Aug. 2024.
doi: 10.6028/NIST.FIPS.204.

\bibitem{rfc8032}
S. Josefsson and I. Liusvaara,
``Edwards-Curve Digital Signature Algorithm (EdDSA),''
RFC 8032, Internet Research Task Force, Jan. 2017.
doi: 10.17487/RFC8032.

\bibitem{gmr}
S. Goldwasser, S. Micali, and R. L. Rivest,
``A Digital Signature Scheme Secure Against Adaptive Chosen-Message Attacks,''
\emph{SIAM Journal on Computing},
vol. 17, no. 2, pp. 281--308, 1988.
doi: 10.1137/0217017.

\end{thebibliography}
\end{document}